\documentclass[10pt]{article}
\usepackage{amsmath,amsfonts,amssymb}
\usepackage{empheq}
\usepackage{caption}
\usepackage{url}
\usepackage{hyperref}
\usepackage{xcolor}
\def\qed{\hfill$\square$}
\def\fqed{\eqno{\square}}
\def\wt{\widetilde}
\def\ul{\underline}

\def\mn{{\mathfrak{n}}}

\def\dg{\mathrm{diag}}

\def\SG{{\cal SG}}

\def\wt#1{\widetilde{#1}}

\def\ov#1{\overline{#1}}

\def\mR{{{\mathfrak{R}}}}
\def\mr{{{\mathfrak{r}}}}

\def\mZ{{\mathfrak{Z}}}
\def\mA{{\mathfrak{A}}}
\def\mB{{\mathfrak{B}}}
\def\mX{{\mathfrak{X}}}
\def\mY{{\mathfrak{Y}}}

\def\three?{3}
\def\four?{4}
\def\ten?{10}

\def\Argmin{\mathop{\hbox{\rm Argmin}}}
\def\beq{\begin{equation}}
\def\eeq{\end{equation}}
\newtheorem{observation}{Observation}[section]

\def\norm2to2{{\|\cdot\|_{2,2}}}
\def\Prob{\mathrm{Prob}}

\def\bE{{\mathbf{E}}}

\def\inter{\hbox{\rm  int}}

\def\Diag{\hbox{\rm  Diag}}

\def\mR{{\mathfrak{R}}}

\def\Opt{\hbox{\rm Opt}}

\def\Tr{{\mathop{\hbox{\rm  Tr}}}}
\def\cA{{\cal A}}
\def\cB{{\cal B}}

\def\cH{{\cal H}}

\def\cK{{\cal K}}
\def\cL{{\cal L}}

\def\cN{{\cal N}}

\def\cP{{\cal P}}

\def\cR{{\cal R}}
\def\cS{{\cal S}}
\def\cT{{\cal T}}

\def\cV{{\cal V}}
\def\cW{{\cal W}}
\def\cX{{\cal X}}
\def\cY{{\cal Y}}
\def\cZ{{\cal Z}}

\def\C{{\cal C}}

\def\Argmin{\mathop{\hbox{\rm  Argmin}}}

\def\abs{\mbox\hbox{\rm  abs}}

\def\bS{{\mathbf{S}}}

\def\abs{{\hbox{\rm abs}}}

\def\mypict3{\epsfxsize=220pt\epsfysize=80pt\epsfbox}

\def\crd{\color{red}}

\def\bR{{\mathbf{R}}}

\def\cH{{\cal H}}
\def\Col{{\mathrm{Col}}}

\def\Risk{{\hbox{\rm Risk}}}

\newcommand{\be}{\begin{eqnarray}}
\newcommand{\ee}[1]{\label{#1}\end{eqnarray}}
\newcommand{\nn}{\nonumber \\}
\newcommand{\ese}{\end{align*}}
\newcommand{\bse}{\begin{align*}}
\newcommand{\rf}[1]{~(\ref{#1})}
\newcommand{\wh}[1]{{\widehat{#1}}}
\def\mR{{\mathfrak{R}}}

\newcommand{\hide}[1]{{}}

\newcommand{\anc}[2]{{\color{violet} #2}}
\newtheorem{lemma}{Lemma}[section]
\newtheorem{proposition}{Proposition}[section]
\definecolor{MyDarkBlue}{rgb}{0,0.08,0.45}
\def\mcP{{\mathfrak{P}}}
\newcommand{\half}{\tfrac{1}{2}}
\newcommand{\aic}[2]{{\color{violet}~#2}}
\newcommand{\Antimonotonicity}{{{$H$-monotonicity}}}
\makeatletter
\newcommand{\manuallabel}[2]{\protected@write\@auxout{}{\string\newlabel{#1}{{#2}{\thepage}}}}
\makeatother

\title{Recovering linear images of sparse signals from indirect observations}
\author{
Anatoli Juditsky
\thanks{\scriptsize LJK, Universit\'e Grenoble Alpes, 700 Avenue Centrale,  38401 Domaine Universitaire de Saint-Martin-d'Hères, France
$^*${\tt anatoli.juditsky@univ-grenoble-alpes.fr}}
\and Arkadi Nemirovski
\thanks{\scriptsize Georgia Institute
 of Technology, Atlanta, Georgia
30332, USA, {\tt nemirovs@isye.gatech.edu}}}
\date{}
\begin{document}
\maketitle
\begin{abstract}
In this paper, we develop and analyze techniques for recovering a linear image $Bx$ of an unknown signal $x$ from indirect noisy  observation $\omega=Ax+\xi$. It is {\em a priori} known that $x\in \cX$, a given convex compact set, and that $x$ is $s$-sparse---has at most $s$ nonvanishing entries. The proposed estimates belong to a large family of recovery routines by $\ell_1$-minimization. However, unlike the classical result describing performance of such estimates, we do not make any special (and hard to check) assumptions about the sensing matrix $A$ such as nullspace or Restricted Isometry condition and the like. As a consequence, parameters of the estimates and the upper bounds on their risks are not available in a closed analytic form, but are delivered instead by efficient computation as solutions to explicit convex optimization problems. \end{abstract}

\section{Introduction}\label{sect1}
\paragraph{The problem} we are interested in this work is as follows:
\begin{quotation} \noindent
($\mcP$) Given {\em sensing matrix} $A\in\bR^{m\times n}$, and observation
\begin{equation}\label{obs}
\omega=Ax+\xi_x\qquad\qquad\hbox{[$\xi_x$: observation noise]}
\end{equation}
of unknown signal $x$ known
\begin{enumerate}
\item[$\mcP.1$] to belong to a given  convex compact set $\cX\subset\bR^n$, and
\item[$\mcP.2$] to possess certain sparsity properties, specifically, to have $s$-sparse (i.e., with at most $s$ nonzero entries) image $Cx$,
with known $s$ and $C=[c_1,..,c_p]^T\in\bR^{p\times n}$,
\end{enumerate}
 we aim at
\begin{itemize}
\item[--] recovering the value at $x$ of a given linear form $z\mapsto g(z)=g^Tz$ (problem $(\mcP.g)$)
\item[--] estimating the image $Bx\in \bR^\nu$ of $x$ under a given linear mapping (problem $(\mcP.B$).
\end{itemize}
\end{quotation}
\noindent
To streamline the presentation in the introduction,
\footnote{In ``the main body'' of the paper, along with sub-Gaussian, we also consider general $C$ and other observation models, e.g.,  discrete and Poisson observations, see Section \ref{sec:obss}.} we consider here the situation in which the signal $x$ is sparse (i.e., $C$ is the identity) and the observation noise $\xi_x$ is sub-Gaussian,\footnote{We say that the random vector $\xi\in \bR^m$ is sub-Gaussian with parameters $\mu\in\bR^m$ and $\Sigma\in\bS^+_m$, denoted $\xi\sim \SG(\mu,\Sigma)$, if $\bE\left\{e^{t^T\xi}\right\}
\leq \exp\big\{t^T\mu+\half t^T\Sigma t\big\}$ for all $t\in \bR^m$.} $\xi_x\sim \SG(0,\sigma^2I)$.
In this paper we consider the estimate $\wh x=\wh x_H(\omega)$ of $x$ yielded by $\ell_1$-minimization which in the present situation takes the form \begin{equation}\label{iestx}
\widehat{x}=\widehat{x}_H(\omega)\in\Argmin_{v\in\cX} \left\{\|v\|_1: \|H^T(\omega-Av)\|_\infty\leq 1\right\}
\end{equation}
of $x$ (cf. Dantzig selector), and then recovers $g^Tx$ ($Bx$)  as $g^T\widehat{x}_H(\omega)$ (resp., $B\widehat{x}_H(\omega)$). Here the {\em contrast matrix} $H\in\bR^{m\times M}$ is the principal parameter of the construction. Our goal is to build, in a computationally efficient way, upper confidence bounds on the recovery loss
$|g^T\widehat{x}_H(\omega)-g^Tx|$ (or $\|B\widehat{x}_H(\omega)-Bx\|$ where $\|\cdot\|$ is a given norm).
\paragraph{Related work.}
The problem of estimating linear functionals of a sparse signal from noisy observation \rf{obs} has received much attention in statistical literature with $\ell_1$-minimization techniques such as Lasso and Dantzig Selector \cite{CT1,efron2004,HT1996}, providing the basis of the estimation routines. However, $\ell_1$-minimization is considered unsuitable to be
 used directly for the corresponding inference because of the bias incurred by the corresponding estimates. To
overcome this issue, inferences based on debiasing (or ``desparsifying'') the Lasso have been developed in \cite{javanmard2014confidence,javanmard2014hypothesis,li2020debiasing,van2014asymptotically,zhang2014confidence} which allowed constructing asymptotically normal estimates of individual entries of unknown $x$ with optimal covariance. This approach was then extended to build  asymptotically minimax optimal and adaptive estimates of general linear functionals \cite{cai2017confidence}, see also the detailed survey  of the present state-of-art \cite{cai2023statistical}. We can also mention a highly influential parallel framework which is referred to as double/debiased machine learning or post-selection inference for linear functionals (like treatment effects) independently developed in the econometrics literature \cite{belloni2014inference,chernozhukov2018double}.

Available results on estimating linear forms of sparse signals and deriving the corresponding confidence bounds, as well as general results
on performance of estimators based on $\ell_1$-minimization rely on specific hypotheses about sensing (regressor) matrix $A$. Classical assumptions of this kind include Restricted Isometry property (RIP) of $A$ \cite{CT1A}, Restricted Eigenvalue (RE) condition \cite{bickel2009simultaneous}, or Compatibility condition \cite{van2009conditions} which are operational when proving accuracy bounds for Lasso and Dantzig Selector estimates.\footnote{A simply-looking {\em necessary and sufficient condition} for the matrix $A$ to be {\em $s$-good}, that is, such that a minimizer of $\|v\|_1$ under the linear constraints $Av=Ax$ recovers exactly every $s$-sparse signal
 $x$, is provided by  the {\em nullspace property} \cite{DonohoHuo}. This condition underlies all accuracy guarantee for recovery by Lasso or Dantzig Selector for ``individual" matrices $A$  known to us.} as well as precision guarantees for debiased estimates \cite{van2014asymptotically}. It is known \cite{CT1A} that such conditions hold  with overwhelming probability in the large range of sparsity parameter $s$ for large random sensing matrices.

 Another line of results on the properties of such estimates, \cite{javanmard2014confidence,javanmard2014hypothesis,cai2017confidence} and construction of confidence bounds for sparse recovery \cite{dezeure2017high,nickl2014confidence,zhang2017simultaneous} rely on the explicit hypothesis of random nature of the regressor matrix $A$, e.g., on the assumption that $A$ is drawn from the Gaussian ensemble, that resampling from the same distribution is possible, etc.

The difficulty here is that available confidence bounds for linear functional and signal recovery are not valid unless
RIP, or RE, or Compatibility, or, at the least, nullspace condition holds, and these conditions cannot be efficiently verified (which essentially means that such assumptions cannot be checked even for an individual matrix $A$ of a moderate size). On the other hand, all presently known verifiable conditions (e.g., Mutual Incoherence (MI) condition) are quite conservative.

\paragraph{Our contribution.}  Our objective in this paper is to derive the estimates of linear functionals and linear images of unknown $x$ without making any assumptions about the matrix $A$.
This is done through building {\em a computationally tractable outer approximation $\ov\mZ_H$} for the ``difficult'' set $\mZ$ which localizes the recovery error $\zeta=\wh x-x$ with ``high probability." The principles of construction of such approximation are explained in detail in Section \ref{Sec1} and rely upon the ideas developed in \cite{juditsky2011verifiable}. This approximation allows to establish efficiently computable error bounds by maximizing the value of $g^Tz$ over $z\in \mZ_H$; it also allows to solve efficiently the problem of contrast optimization by minimizing these bounds with respect to matrix $H$ of contrasts. The same approximation is then used to derive the accuracy guarantees and optimize the contrast construction for the recovery of the sparse signal and its image.

Because we do not make any specific assumptions about matrix $A$ and the geometry of the set $\cX$, our results are of the ``purely computational nature." The latter means that the bounds for the risk we provide as well as the essential estimate parameter---contrast matrix $H$---are not available in the closed analytical form of ``rates of convergence" usually provided to quantify the estimation accuracy. On the other hand, given the problem data---matrix $A$, signal set $\cX$ and the information about the family of noise distributions, the optimized contrast $H$ and the corresponding risk bound are readily available ``by computation." In this perspective, the present work continues the line of research on ``operational approach'' to statistical estimation initiated by the work \cite{Don95}.

\paragraph{Organization of the paper.} We start in Section \ref{Sec1} with deriving efficiently computable confidence bounds for the error of estimating a linear functional $g^Tx$ of unknown sparse $x$, and then optimize these bounds with respect to the available problem data---sensing matrix $A$, the signal set $\cX$ and available information about the distribution of the noise $\xi_x$. Section \ref{sec:3} deals with the problem of the sparse signal recovery. Here we discuss the approach to building the accuracy guarantees and optimal contrast synthesis  based on the results of Section \ref{Sec1} which allows handling the case of an arbitrary convex compact signal set $\cX$. Then in Section \ref{Sec2b} we expose an alternative approach which deals with the situation in which the symmetrization $\cX-\cX$ of the signal set is assumed to be a basic ellitope (see Section \ref{Sec2b} for definition), e.g., finite and bounded intersections of ellipsoids/elliptic cylinders centered at the origin, or $\|\cdot\|_p$-balls, $2\leq p\leq\infty$. It should be added that aside from the just outlined case of sub-Gaussian observation, the derived results are valid also in the case of discrete and Poisson repeated observations described in Section \ref{sec:obss}.
\section{Preliminaries}\label{sectprelim}
We start with specifying some components of the stated in the introduction problem of our interest.

\subsection{Observation schemes}\label{sec:obss} We allow the distribution $P_x$ of the observation noise $\xi_x$   to depend on $x$  and assume that for every $x\in\cX$ it belongs to some family $\cP_x$. The prior information on $\cP_x$ we have is as follows: we are given a family of norms $\pi_\delta(\cdot)$ on $\bR^m$  parameterized by $\delta\in(0,1)$ such that
for every {$x\in\cX$ and} $P_x\in\cP_x$  it holds
\begin{equation}\label{eq1}
\forall  (\delta\in  (0,1),\,x\in\cX,\,h\in\bR^m):\quad \Prob_{\xi_x\sim P_x}\{\xi:\,|h^T\xi_x|>\pi_\delta(h)\}  \leq\delta,
\end{equation}
and $\pi_\delta(\cdot)$ is nonincreasing in $\delta$.
\paragraph{Examples.}
\begin{enumerate}
\item {\sl Uncertain-but-bounded noise:} $\cP$ is composed of all distributions on $\bR^m$ supported on a compact {symmetric   w.r.t. the origin} convex set $\cN\subset \bR^m$.\\
    Here for all $\delta\in(0,1)$ and $h\in\bR^m$
    $
    \pi_\delta(x,\cdot)=\pi(\cdot)
    $
    is the support function of $\cN$, \[\pi(h)=\sup_{\eta\in \cN} \eta^Th.\]
    \item {\sl Gaussian noise:} ${\cP_x}=\{\cN(0,\rho^2I_m),\;0\leq \rho\leq\sigma\}$.
    In this case we set
    \[
    \pi_{\delta}(h) ={\sigma\chi_{\delta}}\|h\|_2\quad \forall x\in \cX
    \]
    where $\chi_\alpha$ is the {\em inverse complementary error function}, i.e., the $(1-\alpha/2)$-quantile of $\cN(0,1)$.
    \item {\sl Sub-Gaussian noise:} $\cP_x\subset\cP_\sigma$, where  $\cP_\sigma$ is the family $\SG(0,\rho^2I_m)$ of sub-Gaussian distributions with parameters $(0, \rho^2I_m)$ with $\rho\in[0,\sigma]$. As is immediately seen, here one can set
         \[
    \pi_{\delta}(h) =\sigma\sqrt{2\ln(2/\delta)}\|h\|_2.
    \]
    \item {\sl Sub-Gaussian mixtures.} Assume that $\cX\subset\{x\in\bR^n_+:\sum_ix_i=1\}$ and that we are given $\sigma>0$, $\mu_i\in\bR^m$, $\rho_i,\;0\leq \rho_i\leq \sigma$, $i=1,...,n$, and $n$ sub-Gaussian distributions $P_i$  with parameters $\mu_i$ and $\rho_i^2I_m$. The realization $\omega$ of the observation $\omega_x\sim P_x$  associated with $x\in\cX$ is generated  as follows: we pick a random index $\iota\in\{1,...,n\}$, the probability to pick a  particular $i$  being $x_i$, and draw $\omega$ from the distribution $P_\iota$.
        \par
        As is immediately seen, the situation fits \rf{obs} with $A=[\mu_1,...,\mu_n]$ and zero-mean $\xi_x=\omega-Ax$.
        \par
        As is shown in the appendix (cf. Lemma \ref{lem:a1}), for properly selected $\varrho=\varrho(\cX,\mu_1,...,\mu_n)$ (e.g., $\varrho={2\over\sqrt{3}}\max_{i,j}\|\mu_i-\mu_j\|_2$), for every $x\in\cX$,  $\xi_x$   is sub-Gaussian with parameters $0$ and $[\sigma^2+\varrho^2]I_m$. In other words, the case of the sub-Gaussian mixture noise reduces to the sub-Gaussian one.
   \item {\sl Poisson observations.} Suppose that $A$ and $\cX$  satisfy $Ax\geq0$ for all $x\in\cX$. The $i$-th entry of the observation $\omega$ is drawn, independently across $i\leq m$, from Poisson distribution with parameter $[Ax]_i$,  that is, entries of $\xi_x$  are the differences $[\omega-Ax]_i$. As shown in \cite[Section 5.4.1]{PUP}, here one can set
\be \pi_{\delta}(h)=\left(4\ln\left(\tfrac{2}{\delta}\right)\max_{x\in \cX}
\sum_i[Ax]_ih_i^2+\tfrac{{16}}{9}\ln^2\left(\tfrac{2}{\delta}\right)\|h\|_\infty^2\right)^{1/2}.
\ee{Poisson_1}
        \end{enumerate}
        {\bf Repeated observations.}
        As stated so far, the problem of interest is to recover a linear form of $x$ via a {\sl single} observation (\ref{obs}) stemming from $x$. It makes sense to consider also the situation where we observe a collection
    \begin{equation}\label{ind}
        \omega_k=Ax+\xi_x^k,\,1\leq k\leq \cK,
        \end{equation}
        of $\cK$ individual observations, with independent across $k$ noises $\xi_k^i\sim P_x\in\cP_\sigma$.
        We propose to reduce this situation to a single-observation one, the aggregated observation being
        \begin{equation}\label{aggr}
        \omega={1\over \cK}\sum_{k=1}^\cK\omega_k=Ax+\underbrace{{1\over \cK}\sum_{k=1}^\cK\xi_x^k}_{\xi_x^{(\cK)}}.
        \end{equation}
        Note that in the Gaussian and sub-Gaussian cases with noise intensity $\sigma$ in individual observations
        this aggregation keeps the situation Gaussian/sub-Gaussian, the noise intensity in the aggregated observation being $\sigma/\sqrt{\cK}$. Similarly, by the stability property of the Poisson distribution, in the case of the Poisson observation, passing from the ``individual'' to $\cK$-repeated observation amounts to replacing \rf{Poisson_1} with
          \[
          \pi^{\cK}_{\delta}(h)=\frac{1}\cK\left(4\ln\left(\tfrac{2}{\delta}\right)\cK\max_{x\in \cX}\sum_i[Ax]_ih_i^2
          +\tfrac{16}{9}\ln^2\left(\tfrac{2}{\delta}\right)\|h\|_\infty^2\right)^{1/2}.
          \]
          Besides this, repeated observations allow to handle another important observation scheme.
        \begin{enumerate}
        \item[6] {\sl Discrete observations.} Here individual observation $\omega_k$ is drawn from a probability distribution on an $m$-element set $Y$, with the $i$-th element drawn with probability $[Ax]_i$. Needless to say, $A$ and $\cX$ are assumed to ensure that $Ax$ is a probability vector (i.e., nonnegative with entries summing up to 1) for every $x\in\cX$. Encoding the $m$ elements of $Y$ by the vectors $e_1,...,e_m$ of the canonic basis of $\bR^m$, the $k$-th observation takes the form
            \[
            \omega_k=Ax+\xi_x^k,
            \]
            with independent across $k$ zero mean noises $\xi_x^k$ taking values $e_i-Ax$ with probabilities $[Ax]_i$.
            The resulting ``single-observation'' scheme obeys  (\ref{eq1}) with $\pi_{\delta}(h)=\pi^\cK_{\delta}(h)$ (cf. \cite[Section 5.1.3.2]{PUP}),
            \[
            \pi^\cK_{\delta}(h)=\frac{1}\cK\left(4\ln(\tfrac{2}\delta)\cK\max_{x\in \cX}
            \sum_i[Ax]_ih_i^2+\tfrac{64}9\ln^2(\tfrac{2}\delta)\|h\|_\infty^2\right)^{1/2}
            \]
                           \end{enumerate}

\par\noindent{\bf Risk.}
Given $\epsilon\in(0,1)$, we quantify the performance of a candidate estimate $\widehat{g}(\omega)$ by its $(s,\epsilon)$-risk
\[
\Risk_\epsilon[\widehat{g}|\cX^s]=\inf\{\rho: \Prob_{\xi_x\sim P_x}
\{|\widehat{g}(Ax+\xi_x)-g^Tx|>\rho\}\leq\epsilon\;\forall x\in\cX^s\}
\]
where $\cX^s$ is the set of all $x\in\cX$ with $s$-sparse vectors $Cx$.

\subsection{Polyhedral estimates}\label{sect11}
Consider the situation obtained from the one described in the beginning of the introduction section by
\begin{itemize} \item lifting the sparsity assumption: the observed signal $x$ can be a whatever point of  a given convex compact set $\cX$
\item passing from recovering a linear form $g^Tx$ to recovering the image $Gx$ of $x$  under a given linear mapping $x\mapsto Gx:\bR^n\to\bR^\nu$, the performance of a candidate estimate $\widehat{G}(\cdot):\bR^m\to\bR^\mu$ quantified by its $(\epsilon,\|\cdot\|)$-risk
    $$
    \Risk_{\epsilon,\|\cdot\|}[\widehat{G}|\cX]=\inf\left\{\rho:\Prob_{\xi_x\sim P_x}\{\|\widehat{G}(Ax+\xi_x)-Gx\| >\rho\}\leq\epsilon\,\forall x\in \cX\right\}
    $$
    where  $\|\cdot\|$ is a given norm on $\bR^\nu$.
\end{itemize}
One of the approaches to the resulting estimation problem is to use a properly designed {\sl polyhedral estimate} introduced in \cite{juditsky2020polyhedral,PUP}.
\par Let us fix the risk tolerance $\epsilon\in(0,1)$.
\paragraph{Polyhedral estimates.} Given an $m\times  M$ matrix $H$ and $x\in\cX$, we set
$$
\Xi[x,H]=\{\xi\in\bR^m: \|H^T\xi\|_\infty\leq1\}.
$$
We call an $m\times M$ matrix $H$ {\sl $(1-\epsilon)$-admissible},  {if its columns $h_j$ satisfy $\pi_{\epsilon/M}(h_j)\leq1$, $j=1,...,M$.} Note  that due to the origin of $\pi_\delta$ and the union bound,
\begin{quotation}\noindent
(!) {\sl  When $H$ is $(1-\epsilon)$-admissible, for every $x\in\cX$ the event $\xi_x\in\Xi[x,H]$ takes place with $P_x$-probability at least $1-\epsilon$.}
\end{quotation}
Given $(1-\epsilon)$-admissible $m\times M$ matrix $H$ (referred to as the {\sl contrast matrix}), we associate with it {\sl polyhedral estimate $\wh x_H(\cdot)$ of signal $x$  underlying observation (\ref{obs})} defined as
\begin{equation}\label{polyest}
\wh x_H(\omega)\in\Argmin_u \left\{\|Cu\|_1:u\in\cV[\omega]:=\{u\in\cX: \|H^T[\omega-Au]\|_\infty \right\}
\end{equation}
(the estimate is undefined when $\cV[\omega]=\emptyset$). This estimate gives rise to estimates $\wh{G}_H(\cdot)=G\wh x_H(\cdot)$ of linear images $Gx$ of the signal $x$ underlying observations, and we refer to these estimates as {\sl associated with contrast matrix $H$ polyhedral estimates} of the images of $x$.
\par
Our interest in polyhedral estimates stems from the fact that in many interesting cases they are provably near-minimax-optimal. Specifically, assuming the observation noise to be Gaussian and independent of the signal
\begin{itemize}
\item it is shown in  \cite{juditsky2020polyhedral,PUP} that for a rather wide variety of signal sets $\cX$ (e.g., intersections of  $K$ centered at the origin ellipsoids/elliptic cylinders) and norms $\|\cdot\|$ (e.g., for $\|\cdot\|_p$-norms with $p\leq 2$),  given a contrast matrix $H$, one can efficiently compute
     "presumably tight'' upper bound $\mr[H]$ on the $(\epsilon,\|\cdot\|)$-risk of the estimate $\wh G_H$  (the risk  itself usually is difficult to compute). This bound  can be efficiently optimized over $(1-\epsilon)$-admissible  contrast matrices $H$, and the resulting bound is within  a ``moderate''---logarithmic in $K$ and $1/\epsilon$---factor of the minimax risk;
    \item from the results of \cite{Don95,linform} it follows that {\sl when recovering a linear form}, for every convex compact set $\cX$ which is computationally tractable, one can compute efficiently a {\sl single-column} contrast $H$  in such a way that the corresponding risk bound $\mr$ is within an absolute constant factor of the minimax risk.
        \end{itemize}
\subsection{Starting points}\label{Sec0}
We now introduce some notation and make several observations that will be used throughout the remainder of the paper.
\paragraph{\ref{Sec0}.1. Notation.} In the sequel
\begin{enumerate}
\item  $\|y\|_{s,p}$, $p\in[1,\infty]$, stands for the norm on $\bR^n$ defined as the $\ell_p$-norm of the vector composed of $s\leq n$ largest in magnitude entries of $y\in\bR^n$; (e.g., $\|y\|_{s,1}$ is the sum of magnitudes of the $s$ largest in magnitude entries in $y$).
\item  $\ov\cX=\cX-\cX$, $\cX^s=\left\{x\in\cX:\,\hbox{$x$ is $s$-sparse}\right\}$

\item We set
\[
\ul\cZ:=\{z\in\ov\cX: \|Cz\|_1\leq 2\|Cz\|_{s,1}\}
\subset \cZ:=\bigcup\limits_{\ell\leq p}\left[\cZ_\ell^+ \cup,\cZ_\ell^-\right]
\]
where for $\ell=1,...,p$
\begin{equation}\label{sets}
\begin{array}{rcl}
\cZ_\ell^+&=&\left\{z\in\ov\cX: \,\|Cz\|_1\leq 2s[Cz]_\ell,\,|[Cz]_j|\leq [Cz]_\ell\;\forall j \right\},\\
\cZ_\ell^-&=&\left\{z\in\ov\cX: \,\|Cz\|_1\leq -2s[Cz]_\ell,\,|[Cz]_j|\leq -[Cz]_\ell\;\forall j \right\}.
\end{array}
 \end{equation}
 \item Let $g\in\bR^n$. Given an $m\times M$ matrix $H$, we set for $\ell\leq p,\,\chi\in\{+,-\}$
 \begin{equation}\label{ZofH}
 \cZ_\ell^\pm[H]=\{z\in\cZ_\ell^\pm: \|H^TAz\|_\infty\leq 2\},
 \end{equation}
\be
\mr_{\ell,\chi}[g,H]&=&\max_{z}\{g^Tz:\,z\in\cZ_\ell^\chi[H]\},
\ee{mrofgHa}
and
 \begin{align}
\mr_{\ell}[g,H]&=\max\left[\mr_{\ell,-}[g,H],\,\mr_{\ell,+}[g,H]\right],\nn
\mr[g,H]&=\max_{\ell}\mr_{\ell}[g,H]\label{mrofgH}
\end{align}
For $G\in\bR^{J\times n}$ with rows $g_1^T,...,g_J^T$ we set
\begin{align}
\mr_{\ell j}^{\pm}[G,H]&=\mr_{\ell}^{\pm}[g_j,H]=\max_z\left\{[Gz]_j:\,z\in\cZ_\ell^\pm,\,\|H^TAz\|_\infty\leq 2\right\},\quad j\leq J,\nn
\mr_{\ell j}[G,H]&=\max\left[r_{\ell j}^{-}[Q,H],\,r_{\ell j}^{+}[Q,H]\right], \quad\ell\leq p,\,j\leq J,\nn
\mr^{(\ell)}[G,H]&=\left[\mr_{\ell 1}[G,H];...;\mr_{\ell J}[G,H]\right],\quad\ell\leq p, \label{mrofGH}\\
\mr_\ell[G,H]&=\max_{j\leq J}\mr_{\ell j}[G,H],\quad\ell\leq p,\nn
\mr[G,H]&=\max_{\ell\leq p}\mr_\ell[G,H].\nonumber
\end{align}
\end{enumerate}

\paragraph{\ref{Sec0}.2} Our first observation is as follows:
\begin{observation}\label{simpleobs}

Let $s\in\{1,...,p\}$, $H$ be a $(1-\epsilon)$-admissible $m\times M$  contrast matrix,
$x\in\cX^s$ and $\xi_x\in \Xi[x,H]$. Then $\wh x=\wh x_H(Ax+\xi_x)$ is well defined, and the recovery error $\zeta=\wh x -x$ belongs to $\underline\cZ$ and satisfies $\|H^TA\zeta\|_\infty\leq2$, that is,
\[
\zeta\in\ul\cZ[H]:=\left\{z\in\ul\cZ:\|H^TAz\|_\infty\leq 2\right\},
\]
so that
\be
\zeta\in\mZ[H]:=\bigcup\limits_{\ell\leq p}\left[\cZ_\ell^+[H]\cup\cZ_\ell^-[H]\right]
\ee{thatis}
Besides this, one has
\begin{equation}\label{eq14}
\|C\zeta\|_p\leq 2^{1/p}\|C\zeta\|_{s,p}\leq (2s)^{1/p}\|C\zeta\|_\infty,\quad1\leq p\leq\infty.
\end{equation}
Finally, the sets $\cZ_\ell^\pm[H]$ ``are decreasing as $H$ grows": when $H'$ is a submatrix of $H$, we have $\cZ_\ell^\pm[H]\subset \cZ_\ell^\pm[H']$; in the sequel, we refer to this property as {\em \Antimonotonicity.}
\end{observation}
{\bf Proof.} Under the premise of the observation, $x\in \cV[Ax+\xi_x]$, implying that the latter set is nonempty,  whence $\wh x$ is well defined and belongs to $\cV[Ax+\xi_x]$. Thus, under the circumstances, $\zeta\in\cV[Ax+\xi_x]-\cV[Ax+\xi_x]\subset \cX-\cX$ and therefore
\begin{equation}\label{eq13}
\zeta\in\ov\cX\;\mbox{and}\;\|H^TA\zeta\|_\infty\leq 2.
\end{equation}
Besides this, $y:=Cx$ is $s$-sparse and whenever $\wh x$ is well defined, for  $\wh y:=C\wh x$ one has $\|\wh y\|_1\leq \|y\|_1$. By Lemma \ref{simple1} these
relations imply (\ref{eq14}). The latter relation with $p=1$ means that
\[\zeta\in \ul\cZ\cap \{z:\,\|H^TAz\|_\infty\leq 2\}\subset\{z\in\ov\cX:\,\|Cz\|_1\leq 2s\|Cz\|_\infty,\,\|H^TAz\|_\infty\leq 2\}.
\]
Now, let $\ell\in\{1,...,p\}$ be the index of the largest in magnitude entry in $C\zeta$,
and $\chi\in\{-,+\}$ be the sign of this entry. We have
\[
\zeta\in\{z\in\ov\cX: \,\|Cz\|_1\leq 2s\|Cz\|_\infty,\, \|Cz\|_\infty=\chi[Cz]_\ell\}=\cZ_\ell^{\chi}
\] which combines with $\|H^TA\zeta\|_\infty\leq 2$ to imply that
$\zeta\in\cZ_\ell^\chi[H]$. {\Antimonotonicity} is evident. \qed
\paragraph{\ref{Sec0}.3} Our next observation goes back to  Observation 5.3 in \cite[Section 5.1.4.3]{PUP}:
\begin{observation}\label{sk1}
Given
$e\in \bR^n$, $A\in\bR^{m\times n}$, a norm $s(\cdot)$ on $\bR^m$, a convex and compact set $\cW\subset\bR^n$ containing the origin, and $\kappa>0$,
consider the  saddle point problem
\begin{equation}\label{problms}
\Opt=\inf_{f\in\bR^m}\max_{w\in \cW}\left\{\psi(f,w)=[e-A^Tf]^Tw+\kappa s(f)\right\}
\end{equation}
together with the induced primal and dual problems
\begin{align*}
\Opt(P)&=\inf_{f\in\bR^m}\left\{\ov\psi(f)=\max_{w\in \cW}\psi(f,w)\right\}
=\inf_{f\in\bR^m}\left\{\kappa s(f)+\max_{w\in \cW}[e-A^Tf]^Tw\right\},\tag{\em P}
\end{align*}
\begin{align*}
\Opt(D)&=\sup_{w\in\cW}\left\{\underline\psi(w)=\inf_{f\in \bR^m}\psi(f,w)\right\}
=\max_{w\in \cW}\left\{e^Tw+\inf_{f\in\bR^m}\left[\kappa s(f)-w^TA^Tf\right]\right\}.\tag{\em D}
\end{align*}
Both problems $(P)$ and $(D)$ are solvable with equal optimal values: $\Opt=\Opt(P)=\Opt(D)$, and when
$f_*$ is an optimal solution to $(P)$, and $h_*$ is such that $s(h_*)=1$ and $f_*=s(f_*)h_*$
\par{\rm (i)} One has
\begin{equation}\label{eq55}
\Opt\geq\max_w\{e^Tw: \,w\in\cW,\, h_*^TA
w\leq\kappa\}
\end{equation}
\par{\rm (ii)} Whenever $H$ is a matrix with columns $h_j$ satisfying $s(h_j)\leq 1$, one has
\begin{equation}\label{eq666}
\Opt\leq\max_w\{e^Tw: \,w\in\cW,\, \|H^TAw\|_\infty\leq\kappa\}
\end{equation}
\end{observation}
For the sake of completeness, here is the {\bf proof of the observation.}
As
\[
\inf_{f\in\bR^m}[-w^TA^Tf+\kappa s(f)]=\left\{\begin{array}{cl}-\infty,&s^*(Aw)>\kappa,\\
0,&s^*(Aw)\leq1\end{array}\right.
\]
where $s^*(u)=\sup_{f}\{u^Tf:\,s(f)\leq 1\}$ is the conjugate norm of $s(\cdot)$, we have
\begin{equation}\label{eq777}
\Opt(D)=\max_{w\in \cW} \left\{e^Tw:\,s^*(Aw)\leq \kappa\right\}.
\end{equation}
Since $\cW$ is compact we have $ \Opt(P)=\Opt(D)=\Opt$ by the Sion-Kakutani theorem. Besides this, $(D)$ is clearly solvable, and $(P)$ is
solvable as well because $\ov\psi(\cdot)$ is continuous  due to the compactness of $\cW$, and satisfies $\ov\psi(f)\geq s(f)$ as
$0\in\cW$, so that $\ov\psi(\cdot)$ has bounded level sets.\\
Next, let $w$ be a feasible solution to the problem in the right-hand side of \rf{eq55}. We have
\begin{align*}
e^Tw&=\big[(e-A^Tf_*)^Tw+\kappa s(f_*)\big]+[f_*^TAw-\kappa s(f_*)]\\&=
\big[(e-A^Tf_*)^Tw+\kappa s(f_*)\big]+s(f_*)[h_*^TAw-\kappa]\leq \Opt,
\end{align*}
implying \rf{eq55}.
Now let $\ov w$ be an optimal solution to the optimization problem in (\ref{eq777}), so that $s^*(A\ov w)\leq\kappa$,
and therefore $\ov w$, for every $H$ satisfying the premise of (ii), is a feasible solution to the right-hand side problem in (\ref{eq666}). Consequently, the maximum in the right-hand side of
(\ref{eq666}) is at least $e^T\ov w=\Opt(D)=\Opt(P)=\Opt$, and (\ref{eq666}) follows.
\qed
\paragraph{\ref{Sec0}.4}
Recall that we have fixed, once and for all, a tolerance $\epsilon\in(0,1)$. Given  an $m\times M$ $(1-\epsilon)$-admissible  contrast matrix $H$,
we have associated with it  the sets
$\mZ[H]=\bigcup_{\ell\leq p}\cZ_\ell^\pm[H]$, see (\ref{thatis}).
By Observation \ref{simpleobs}, when $x\in\cX^s$ and $\xi_x\in\Xi[x,H]$, the estimate $\widehat{x}=\wh x_H(Ax+\xi_x)$ (see
(\ref{polyest})) is well defined,
and the recovery error $\zeta=\wh x-x$ belongs to $\mZ[H]$,  that is, belongs to one of the sets $\cZ_\ell^\pm[H]$ (see \rf{thatis}). Given $H$  and $x\in \cX^s$,
we say that {\em  case $(\ell,\chi)$ occurs}
if the recovery error $\zeta$ is well defined and belongs to $\cZ_\ell^\chi[H]$. When $\xi_x\in\Xi[x,H]$, by Observation \ref{simpleobs}, this happens if and only if $\ell$ is among the indexes of the largest in magnitude entries of $\zeta$, and $\chi$ is the sign of this entry. Let us denote $\cL$ the set of couples $(\ell, \chi)$, $\ell\leq p$, $\chi\in\{-,+\}$;  by Observation \ref{simpleobs},
\begin{quotation}
(!!) {\sl Whenever $x\in\cX^s$ and $\xi_x\in\Xi[x,H]$, at least one of the cases $(\ell, \chi)\in \cL$ does occur. }
\end{quotation}
We conclude this section with the following
\begin{observation}\label{obsenough} Given a $J\times n$ matrix $Q$ and a $(1-\epsilon)$-admissible $m\times M$ contrast matrix $H$  denote $\wh x=\wh x_H(Ax+\xi_x)$ (see (\ref{polyest})),
 $\wh Q=Q\wh x$, and $\zeta=\wh x -x$. Whenever $x\in\cX^s$ and $\xi_x\in\Xi[x,H]$, assuming that one of the cases $(\ell,\pm)$ occurs for some $\ell=\ell_*$, it holds
 \be
 |[Q\zeta]_j|\leq {\mr}_{\ell_*j}[Q,H],\quad j\leq J.
 \ee{itholds}
(for notation, see \rf{mrofGH}).
 Furthermore, let $\delta\in(0,1)$, and let vectors  $f_{\ell j}^{Q,\delta,\chi}$  be optimal solutions to problems (\ref{problms}) with
 $e=\Col_j[Q^T]$, $B=A$, $\cW=\cZ_\ell^\chi$, and $\kappa=2$.  Denoting by $\rho_{\ell j}^{Q,\delta,\chi}$ the optimal values in these problems,
 let $h_{\ell j}^{Q,\delta,\chi}$ be {such that}
 \[
 \pi_\delta(h_{\ell j}^{Q,\delta,\chi})=1,\quad f_{\ell j}^{Q,\delta,\chi}=\pi_\delta(f_{\ell j}^{Q,\delta,\chi})h_{\ell j}^{Q,\delta,\chi}.
 \]
 Assuming that $H$ contains as columns vectors $h_{\ell_* j}^{Q,\delta,\chi}$, $j\leq J$, $\chi\in\{-,+\}$,
 one has
 \be
 |[Q\zeta]_j|&\leq&\rho_{\ell_*j}[Q,\delta]:=\max\left[\rho_{{\ell_*} j}^{Q,\delta,-},\rho_{{\ell_*} j}^{Q,\delta,+}\right], \;j\leq J.
\ee{eq66}
 \end{observation}
 {\bf Proof.} Under the premise of the observation we have $\zeta
 {\in}[\cZ_{\ell_*}^-[H]\cup\cZ_{\ell_*}^+[H]]$,
 that is, $\zeta\in[\cZ_{\ell_*}^-\cup\cZ_{\ell_*}^+]$ and $\|H^TA\zeta\|_\infty\leq 2$, and as $\cZ_\ell^-=-\cZ_\ell^+$, this implies that
\begin{align*}
 |[Q\zeta]_j|&\leq\max\left[\max_z\{[Qz]_j:\,z\in \cZ_{\ell_*}^-,\,\|H^TAz\|_\infty\leq 2\},
 \,\max_z\{[Qz]_j:\,z\in \cZ_{\ell_*}^+\,\|H^TAz\|_\infty\leq 2\}\right]\\&={\mr}_{\ell_*j}[Q,H], \quad j\leq J,
 \end{align*}
 as required in  \rf{itholds}. To prove the ``furthermore" part of the claim, note that by Observation \ref{sk1}
 we have
 $$
 \rho_{\ell_* j}^{Q,\delta,\chi}\geq \max_z\left\{[Qz]_j:\,z\in \cZ_{\ell_*}^\chi, \,[h_{\ell_*j}^{Q,\delta,\chi}]^TAz\leq 2\right\},\;j\leq J,\,\chi\in\{-,+\},
 $$
 and when $H$ contains $h_{\ell_*j}^{Q,\delta,\chi}$ as a column, the right-hand side in this inequality upper-bounds the quantity
 \[{\mr}_{\ell_*,j}^\chi [Q,H] = \max_z\left\{[Qz]_j:\,z\in \cZ_{\ell_*}^\chi, \,\|H^TAz\|_\infty\leq 2\right\}.
 \]
 Thus, under the premise of the observation, we have $\rho_{\ell_* j}^{Q,\delta,\chi} \geq {\mr}_{\ell_*j}^\chi [Q,H]$, whence
 $\rho_{\ell_* j}[Q,\delta] \geq {\mr}_{\ell_*j} [Q,H]$, and (\ref{eq66}) follows from (\ref{itholds}). \qed

\section{Estimating linear functionals}\label{Sec1}
Our subject in this section is the design of a ``presumably good" polyhedral estimate for the problem $(\mcP.g)$ stated in the introduction.
\subsection{Risk analysis}\label{sect2}
 The following statement is an  immediate consequence of Observation \ref{simpleobs}:
\begin{proposition}\label{obslin}
Given linear form $g^Tx$ on $\bR^n$, sparsity level $s$, and $(1-\epsilon)$-admissible $m\times M$ contrast matrix $H$, let  $x\in\cX^s$ and $\xi_x\in\Xi[x,H]$, and let $\wh g_H=g^T\wh x_H$  be the polyhedral estimate of $g^Tx$ associated with $H$, see (\ref{polyest}).
One has
\begin{equation}\label{linbound}
|\wh g_H(Ax+\xi_x)-g^Tx|=|g^T\underbrace{[\wh x_H(Ax+\xi_x)-x]}_{\zeta}|\leq \mr[g,H]
\end{equation}
(see \rf{mrofgH} for notation), whence
\begin{equation}\label{risklinest}
\Risk_\epsilon[\wh g_H|\cX^s]\leq \mr[g,H].
\end{equation}
\end{proposition}
{\bf Proof.} By Observation \ref{simpleobs}, under the premise of Proposition \ref{obslin}
one has
\[\zeta\in\bigcup_{\ell\leq p}\left[\cZ_\ell^+[H]\cup\cZ_\ell^-[H]\right],
\] see (\ref{thatis}), whence
\begin{align*}
|\wh g_H(Ax+\xi_x)-g^Tx|&\leq\max_{\ell\leq p}\max\left[\max_{z\in \cZ_\ell^+[H]}|g^Tz|,\,\max_{z\in \cZ_\ell^-[H]}|g^Tz|\right]\\
&=
\max_{\ell}\max\left[\max_{z\in \cZ_\ell^+[H]}g^Tz,\,\max_{z\in \cZ_\ell^-[H]}g^Tz\right]=\mr[g,H]
\end{align*}
where the first equality is due to $\cZ_\ell^-=-\cZ_\ell^+$; we have verified (\ref{linbound}). Invoking (!) completes the proof. \qed
\subsection{Designing the estimate}\label{sec:designlin1}
 Given a $(1-\epsilon)$-admissible contrast matrix $H$, the $(s,\epsilon)$-risk of the associated polyhedral estimate of $g^Tx$ is upper-bounded by the right-hand side of (\ref{risklinest}). To design the estimate, we optimize this bound over $m\times M$  contrast matrices with columns $h_j$ satisfying $\pi_{\epsilon/M}(h_j)\leq1$ (and therefore $(1-\epsilon)$-admissible) with $M\geq2p$.
Optimization in $H$ is implemented using the following construction going back to \cite[Section 5.1]{PUP}.
\begin{quotation}
\noindent Given
$b\in \bR^n$, $B\in\bR^{m\times n}$, a norm $s(\cdot)$ on $\bR^m$, a convex and compact set $\cW\subset\bR^n$ containing the origin, and $\kappa>0$, our recipe for building $H$ is as follows:
we set
$$
\delta=\epsilon/(2p)
$$
and solve $2p$ convex optimization problems
$$
\Opt_\ell^\pm[g]=\min_f\left\{2\pi_\delta(f)+\max_z\{z^T[g-A^Tf]:z\in\cZ_\ell^\pm\}\right\}.
$$
Denoting the respective optimal solutions $f_\ell^\pm$, we find $\pi_\delta$-unit vectors $h_\ell^\pm$
such that $f_\ell^\pm=\pi_\delta(f_\ell^\pm)h_\ell^\pm$ and specify $H$ as the matrix $\ov H$ with the $2p$
columns $h_\ell^-$, $h_\ell^+$, $\ell\leq p$, and claim that the associated estimate $\widehat{g}_{\ov H}$ satisfies the risk bound
\begin{equation}\label{claimthat}
\Risk_\epsilon[\widehat{g}_{\ov H}|\cX^s]\leq \Opt[g]:=\max_{\ell\leq p}\max\left[\Opt_\ell^+[g],\,\Opt_\ell^+[g]\right].
\end{equation}
\end{quotation}
To justify (\ref{claimthat}), note that by Observation \ref{sk1} one has
$\mr_{\ell,\chi}[g,[h_\ell^\pm]]\leq \Opt_\ell^\pm[g]$ for all $\ell$, whence $\mr_{\ell,\chi}[g,\ov H]\leq \Opt_\ell^\chi[g]$.
As a result, bound (\ref{risklinest}) as applied with  $H=\ov H$ implies (\ref{claimthat}).
\par
Note that invoking Lemma \ref{sk1}, we immediately see that the proposed design optimizes the right-hand side of \rf{risklinest} over all contrast matrices with columns $h_j$ satisfying $\pi_{\epsilon/(2p)}(h_j)\leq1$. Taking into account that, for our observation schemes, $\pi_\delta$ is ``nearly independent'' of $\delta$, we conclude that our design is as good as it can be -- provided that we quantify the performance of a polyhedral estimate by the efficiently computable upper bound \rf{risklinest} on its risk rather than by its ``computationally intractable" actual risk.

\subsection{Linearly corrected (debiased) estimation}

Consider the situation as follows. Suppose that we are given a contrast matrix $H\in\bR^{m\times M}$ which is $(1-\varepsilon)$-admissible for some $\varepsilon\in(0,1)$, and {let} $\wh x=\wh x_H(\omega)$ {be} the polyhedral estimate \rf{polyest} of
 $x\in \cX^s$ {underlying} observation $\omega=Ax+\xi_x$. Given $g\in \bR^{n}$ we want to estimate $g^Tx$, and our objective is to ``improve" the biased ``plug-in" estimate $g^T\wh x$.
This {can} be done, for instance, by adding to  $g^T\wh x_H(\omega)$ a correction term which is linear in the observation, thus arriving at the estimate of the form
\be
\wh g_{(H,f)}(\omega)=g^T\wh x+f^T(\omega-A\wh x)
\ee{linplus}
where $f\in\bR^{m}$.

Let (cf. \rf{mrofgHa})
\[\mr=\max_{\ell\leq p}\mr_{\ell,\pm},\quad\mr_{\ell,\pm}=\max_{z} \left\{[Cz]_i:\,z\in \cZ^\pm_\ell,\,\|H^TAz\|_\infty\leq 2\right\},\]
and
\[
\ov\cZ:=\{z\in \ov{\cX}:\,\|Cz\|_\infty\leq
\mr,\,\|Cz\|_{1}\leq 2s\mr\}.
\]
Given $\upsilon\in(0,1)$, consider the optimization problem
\be
\wt\rho^g:=\min_{f\in\bR^m}\left\{\psi^g(f)=\max_{z\in \ov\cZ}[g-A^Tf]^Tz+\pi_{\upsilon}(f)\right\}.
\ee{frho}

\begin{proposition}\label{prop:lincorr} In the situation of this section, the problem in \rf{frho} is solvable; let $\ov f$ be an optimal solution. Let $\epsilon=\varepsilon+\upsilon\in(0,1)$, and let $\wh g_{(H,\ov f)}(\cdot)$ be the estimate defined in \rf{linplus} with $f=\ov f$. Then the $(s,\epsilon)$-risk of $\wh g_{(H,\ov f)}(\cdot)$ satisfies
\[\Risk_\epsilon[\widehat{g}|\cX^s]\leq \wt\rho^g.
\]
\end{proposition}
{\bf Proof.} Given $f\in \bR^{m}$ and $x\in \cX^s$, denote $\Xi_f(x)$ the set of $\xi_x$ such that $|[f^T\xi_x]|\leq \pi_{\upsilon}(f)$; by the definition of $\pi_{\upsilon}(\cdot)$, the $P_x$-probability of $\Xi_f(x)$ is $\geq1-\upsilon$. Let now $x\in \cX^s$ and $\xi_x\in \Xi[x,H]\cap\Xi_f({x})$ be fixed. As we know from Observation \ref{simpleobs} (cf. \rf{eq14} applied with $p=1$),  in this case the error $\zeta=\wh x_H(Ax+\xi_x)-x$ satisfies
\[
\zeta\in\ov\cX,\;\|C\zeta\|_\infty \leq \mr,\;\|C\zeta\|_1\leq 2s\mr,\;\|H^TA\zeta\|_\infty\leq2,
\]
implying, in particular, that $\zeta\in \ov\cZ$.
Thus, the error $\wh g-g^Tx$ of $\wh g=\wh g_{(H,f)}(Ax+\xi_x)$ satisfies
\[
\wh g-g^Tx= [g-A^Tf]^T(\wh x-x)+f^T\xi_{x}\leq \max_{z\in \ov\cZ}[g-A^Tf]^Tz+f^T\xi_x\leq \max_{z\in \ov\cZ}[g-A^Tf]^Tz+\pi_{\upsilon}(f)
\]
(recall that $\xi_x\in \Xi_f[x]$), and also, due to the symmetry of $\ov\cZ$,
\[
\wh g-g^Tx\geq  \min_{z\in \ov\cZ}[g-A^Tf]^Tz-\pi_{\upsilon}(f)= -\left[\max_{z\in \ov\cZ}[g-A^Tf]^Tz+\pi_{\upsilon}(f)\right].
\]
We conclude that
\[|\wh g-g^Tx|\leq \psi^{g}(f):=\max_{z\in {\ov\cZ}}[g-A^Tf]^Tz+\pi_{\upsilon}(f).
\]
Because the set $\Xi[x,H]\cap\Xi_f(x)$ is of $P_x$-probability at least $1-\varepsilon-\upsilon=1-\epsilon$, this implies
that the $(s,\epsilon)$-risk of the estimate $\wh g$ is bounded by $\psi^{g}(f)$.
\par
One easily checks that the function $\psi^{g}(\cdot)$ is continuous and has bounded level sets (for this we refer to the proof of  Observation \ref{sk1}), so the optimization problem in \rf{frho} is indeed solvable and $\ov f$ is well defined. By the above,
the $(s,\epsilon)$-risk of the  estimate $\wh g_{(H,\ov f)}(\cdot)$ is bounded by $\wt\rho^g$.\qed
\paragraph{Remark.}  {The ``linearly corrected estimate'' introduced in this section can be viewed as the counterpart, in our setting, of the debiased estimate of \cite{cai2023statistical,van2014asymptotically}.}

The simple argument at the core of the proof of Proposition \ref{prop:lincorr} is the same which underlies the error bound of the affine estimate when estimating the linear functional $g(x)=g^Tx$ of unknown $x\in \cX$ from the observation
$\omega=Ax+\xi_x$ in the case of Gaussian, Poisson and Discrete observation schemes when $\cX$ is a convex and compact set
\cite{Don95,linform}.  Moreover, when the set $\cX$ is computationally tractable, parameters $f_*$ and $c_*$ of the optimal affine estimate $\wh g_*(\omega)=f_*^T\omega +c_*$ (which minimize the $\epsilon$-risk of the resulting estimate) are obtained by an efficient computation, as solutions to a tractable convex optimization problem. In the latter situation (when $\cX$ is convex), the  affine estimate $\wh g_*$ is also ``nearly minimax optimal,'' meaning that for, e.g., $\epsilon\in(0,\frac{1}4)$, the $\epsilon$-risk of  $\wh g_*$  {\em is close to the minimax estimation risk up to a moderate absolute constant factor.}

\section{Recovering linear image of a sparse signal}\label{sec:3}
In this section, we focus on designing ``presumably good" polyhedral estimates for problem $(\mcP.B)$ of recovering the linear image $Bx$ of the signal $x$ underlying observation  (\ref{obs}).

\subsection{From estimating linear forms to signal estimation}\label{Sec2a}
\subsubsection{Preliminaries}\label{4.1.1}
{Suppose we are given a $J\times n$ matrix $G=[g_1,...,g_J]^T$; our objective is to use the polyhedral estimate to recover the image $Gx$ of the signal $x\in\cX^s$. To this end, given a $(1-\epsilon)$-admissible matrix $H\in \bR^{m\times M}$ we compute $\wh x_H(\omega)$ according to \rf{polyest} and put
$\wh G_H(\omega)=G\wh x_H(\omega)$.

Notice that admissibility of $H$ implies that for all $x\in \cX$ $\Prob_{\xi_x\sim P_x}(\Xi[x, H])\geq 1-\epsilon$.
Now, let $x\in \cX^s$ and $\xi_x\in \Xi[x,H]$, so that $x$ is feasible to \rf{polyest}, and so $\wh x=\wh x_H(\omega)$ is well defined; we put $\zeta=\wh x_H(Ax+\xi_x)-x$.
Now, let $\ell\leq p$. Assuming that one of the cases $(\ell,\pm)$ occurs,  Observation \ref{obsenough} (see \rf{itholds}) implies that
\[\left|[G\zeta]_j\right|\leq  \mr_{\ell j}[G,H],
\]
(for notation, see \rf{mrofGH}), and therefore
\be
\abs\big[G\zeta\big]\leq \mr^{(\ell)}[G,H]
\ee{forlvec}
(we denote $
\abs\big[[x_1;...;x_k]\big]=[|x_1|;...;|x_k|]$).
We conclude that, when estimating $Gx$ by $\wh G_H(Ax+\xi_x)$,
the entries of $\mr^{(\ell)}[G,H]$ are upper bounds on the magnitudes of the coordinates of the estimation error
$G\zeta=\wh G_H(Ax+\xi_x)-Gx$ (and absolute norms of $\mr^{(\ell)}[G,H]$ also bound the corresponding norms of $\wh G_H(Ax+\xi_x)-Gx$).
When combining this observation with (!) and (!!) we conclude that, for any absolute norm $\|\cdot\|$, the $(s,\epsilon,\|\cdot\|)$-risk of the estimate $\wh G_H$ of $Gx$ satisfies
\[\Risk_{\epsilon,\|\cdot\|}[\widehat{G}_H|\cX^s]\leq\max_{\ell\leq p}\left\|\mr^{(\ell)}[G,H]\right\|.
\]}
Next, given  $G=[g_1,...,g_J]^T\in \bR^{J\times n}$ and $\delta\in(0,1)$,  consider $2pJ$ optimization problems
\begin{equation}\label{rhoelljm}
\rho_{\ell j}^{G,\delta,\pm}=\min_f\left\{\max_{z}\left[(g_j-A^Tf)^Tz\right]+{2}\pi_\delta(f):\,z\in \cZ_\ell^\pm\right\}\end{equation}
where $\cZ^\pm_\ell$ are defined in \rf{sets}. We set
\begin{subequations}\label{rhos}
\begin{align}
\rho_{\ell j}[G,\delta]&=\max\left[\rho_{\ell j}^{G,\delta,-},\rho_{\ell j}^{G,\delta,+}\right],\label{rhos.a}\\
\rho^{(\ell)}[G,\delta]&=\left[\rho_{\ell 1}[G,\delta];...;\rho_{\ell J}[G,\delta]\right],\label{rhos.b}\\
\rho_\ell[G,\delta]&=\max_{j}\rho_{\ell,j}[G,\delta],\label{rhos.c}\\
\rho[G,\delta]&=\max_{\ell,j}\rho_{\ell j}[G,\delta].\label{rhos.d}
\end{align}
\end{subequations}
We denote  by $f_{\ell j}^{G,\delta,\pm}$  optimal solutions to problems (\ref{rhoelljm}) and define $h_{\ell j}^{G,\delta,\pm} $
satisfying
\be
 \pi_\delta(h_{\ell j}^{G,\delta,\pm})=1\;\;\mbox{and}\;\;f_{\ell j}^{G,\delta,\pm}=\pi_\delta(f_{\ell j}^{G,\delta,\pm})h_{\ell j}^{G,\delta,\pm}.
 \ee{eq30}

 The following statement is readily implied by  Observation \ref{obsenough}
\begin{proposition}\label{bxinf}
Let $G=[g_1,...,g_J]^T,\; g_j\in\bR^n$, $\epsilon\in(0,1)$,
 signal $x\in\cX^s$ underlying observations, and $\ell\leq p$  be fixed. Suppose that a $(1-\epsilon)$-admissible $m\times M$ contrast matrix $H$ contains, for some $\delta\in(0,1)$,
 a submatrix composed of the $2J$ vectors $h_{\ell j}^{G,\delta,\pm}$, $j\leq J$.

 Then,
 \emph{whenever one of the $(\ell,\pm)$-cases occurs}, the recovery error $\zeta=\widehat{x}_H(Ax+\xi_x)-x$
 of the estimate $\widehat{x}_H$ given by (\ref{polyest})  satisfies
\[
|g_j^T\zeta|\leq\rho_{\ell j}[G,\delta], \quad j\leq J,
\]
and thus for all $j\leq J$
\[
|[G\zeta]_j|\leq \rho_{\ell j}[G,\delta]\leq\rho[G,\delta].
\]
\end{proposition}
Here are some immediate consequences of this result.
\subsubsection{Recovering the sparse signal $y=Cx$}\label{4.1.2}
Applying the derivations of Section \ref{4.1.1} with $G=C$
and taking into account that, by (\ref{eq14}),
whenever $x\in\cX^s$ and  $\xi_x\in\Xi[x,H]$,  the recovery error $\zeta$ by Lemma \ref{simple1} satisfies
\begin{equation}\label{aboutC}
\|C\zeta\|_\theta\leq 2^{1/\theta}\|C\zeta\|_{s,\theta},
\end{equation}
we arrive at the following conclusion:
\begin{proposition}\label{propai11}
Given a $(1-\epsilon)$-admissible $m\times M$ contrast matrix $H$, let $x\in \cX^s$, $\xi_x\in\Xi[x,H]$, and $\ell\leq p$. Then the vector $\wh x_H(Ax+\xi_x)$ (see  \rf{polyest}) is well defined along with $\zeta=\wh x_H(Ax+\xi_x)-x$. Furthermore, if one of the cases $(\ell,\pm)$ occurs then the error $C\zeta$ of the recovery $\widehat{C}_H=C\wh{x}_H(Ax+\xi_x)$ of $Cx$ satisfies (cf. \rf{forlvec})
\[
\abs\big[C\zeta\big]\leq \mr^{(\ell)}[C,H]\quad \mbox{and}\quad \|C\zeta\|_\theta\leq 2^{1/\theta}\|C\zeta\|_{s,\theta},\quad\theta\in[1,\infty],
\]
where $\mr^{(\ell)}[G,H]$ is defined in \rf{mrofGH}. Hence,  for all $\theta\in[1,\infty]$,
\[
\|C\zeta\|_\theta\leq 2^{1/\theta}\left\|\mr^{(\ell)}[C,H]\right\|_{s,\theta},
\]
implying that
\[
\Risk_{\epsilon,\|\cdot\|_\theta}[\widehat{C}_H|\cX^s]\leq\max_{\ell\leq p} 2^{1/\theta}
\left\|\mr^{(\ell)}[C,H]\right\|_{s,\theta}.
\]
\end{proposition}
We also have the following corollary of Proposition \ref{bxinf} (for notation, see \rf{rhos}).
\begin{proposition}\label{propan1}
Suppose that a $(1-\epsilon)$-admissible $m\times M$ contrast matrix $H$ contains, for some $\delta\in(0,1)$, all
columns $h_{\ell j}^{C,\delta,\pm}$,  for $\ell,j\leq p$. Then for every $\ell\leq p$, $x\in \cX^s$, and $\xi_x\in\Xi[x,H]$, assuming that one of the cases $(\ell,\pm)$ occurs, the
error $C\zeta$, of recovering  $Cx$ by $\widehat{C}_H=C\widehat{x}_H(Ax+\xi_x)$ satisfies
\[
\abs[[C\zeta]]\leq\rho^{(\ell)}[C,\delta]
,\]
 so that, by \rf{aboutC}, for all $\theta\in[1,\infty]$ one has
\begin{align*}
\|C\zeta\|_\theta&\leq2^{1/\theta}\|\rho^{(\ell)}[C,\delta]\|_{s,\theta}
\end{align*}
As a result,  one has
$$
\Risk_{\epsilon,\|\cdot\|_\theta}[\widehat{C}_H|\cX^s]\leq\max_{\ell\leq p}
2^{1/\theta}\left\|\rho^{(\ell)}[C,\delta]\right\|_{s,\theta}
$$
\end{proposition}
\subsubsection{Recovering \anc{the image}{} $w=Bx$}\label{4.1.3}
We assume here that the matrix
$B\in\bR^{\nu\times n}$ allows for the decomposition $B=FC$ for some $F\in {\bR}^{\nu\times p}$ and the  recovery error is measured in $\|\cdot\|_2$.
 Given a $(1-\epsilon)$-admissible $H\in\bR^{m\times M}$, let $G=F^TB\in \bR^{p\times n}$ and let $g_1^T,...,g_p^T$
 be the rows of $G$. Let also $x\in \cX^s$, $\xi_x\in \Xi[x,H]$, and  $\zeta=\wh x_{H}(Ax+\xi)-x$.
 Assuming that one of the cases $(\ell,\pm)$ occurs, we conclude by Proposition \ref{obslin} that
$|g_j^T\zeta|\leq \mr_\ell[g_j,H]$, $j\leq J$ (for notation, see (\ref{mrofGH})), so that
\[
\|G\zeta\|_\infty\leq \max_{j\leq p}\mr_\ell[g_j,H]=\mr_\ell[G,H],
\]
and
\[
 |c_j^T\zeta|\leq \mr_\ell[c_j,H],\quad j\leq p
\]
(recall that $c_\ell^T$ are rows of $C$).
When invoking (\ref{aboutC}) with $\theta=1$, we get the second $\leq$ in the following chain:
\begin{align}
\|B\zeta\|_2^2&=(F^TB\zeta)^TC\zeta=(G\zeta)^TC\zeta\leq \|G\zeta\|_\infty\|C\zeta\|_1\leq 2\|G\zeta\|_\infty\|C\zeta\|_{s,1}\nn
&\leq2\mr_\ell[G,H]\left\|\mr^{(\ell)}[C,H]\right\|_{s,1}
\label{eq23}
\end{align}
Recalling that for a $(1-\epsilon)$-admissible $H$ and $x\in\cX^s$ the inclusion $\xi_x\in\Xi[x,H]$ takes place with $P_x$-probability
$\geq 1-\epsilon$, we conclude that
the risk of the associated with contrast $H$ estimate $\wh B_H(\omega)=B\wh x_H(\omega)$ of $Bx$ satisfies the bound
\begin{equation}\label{eq26}
\Risk^2_{\epsilon,\|\cdot\|_2}[\widehat{B}_H|\cX^s]\leq 2\max_{\ell\leq p}\mr_\ell[G,H]\left\|\mr^{(\ell)}[C,H]\right\|_{s,1}.
\end{equation}

We now consider the problem of optimizing this risk bound over the contrast matrix $H$. Here is the underlying construction:
\begin{quotation} \noindent(\#) Given $G=[g_1,...,g_J]^T\in\bR^{J\times n}$ and  $\delta\in(0,1)$,  we compute optimal solutions $f_{\ell j}^{G,\delta,\pm}$  to the optimization
problems in (\ref{rhoelljm}). We then define $h_{\ell j}^{G,\delta,\pm}$ according to (\ref{eq30}). The outcome of the construction is the matrix
$H[{G,\delta}]=[H_1[G,\delta],...,H_p[G,\delta]]$, where $H_\ell[G,\delta]$ is the  $m\times 2J$ matrix with the columns $h_{\ell j}^{G,\delta,\pm}$,
$j\leq J$.
\end{quotation}
As an immediate corollary of
Proposition \ref{bxinf}, we have the following
\begin{observation}\label{verynew} Given $\ell\leq p$, suppose that a $(1-\epsilon)$-admissible  contrast matrix $H$ contains as a submatrix  the matrix $H_\ell[{G,\delta}]$ with some $\delta\in(0,1)$. Assuming that
$x\in \cX^s$, $\xi_x\in\Xi[x,H]$, and one of the cases $(\ell,\pm)$ takes place, the vector $\zeta=\wh x_H(Ax+\xi_x)-x$ satisfies
$$
\abs[G\zeta]\leq\rho^{(\ell)}[G,\delta]\quad\mbox{and}\quad \|G\zeta\|_\infty\leq \rho_\ell[G,\delta]
$$
with $\rho^{(\ell)}[\cdot]$ and $ \rho_\ell[\cdot]$ defined in \rf{rhos.b} and \rf{rhos.c}.
\end{observation}

To optimize the bound (\ref{eq26}) in $H$, we act as follows.
\begin{enumerate}
\item  Given $\delta\in(0,1)$, we apply (\#) to $G=F^TB$, thus obtaining the $m\times 2p^2$ matrix $H[F^TB,\delta]$ such that
\begin{quote}
(A) {\sl The $\pi_\delta$-norms of columns in $H[F^TB,\delta]$ do not exceed 1, and whenever $(1-\epsilon)$-admissible contrast matrix $H$ contains $H[F^TB,\delta]$ as a submatrix,
$x\in\cX^s$, $\xi_x\in\Xi[x,H]$, and one of the cases $(\ell,\pm)$ takes place, $\zeta=\wh x_H(Ax+\xi_x)-x$ satisfies the relation  $\abs[F^TB\zeta] \leq \rho^{(\ell)} [F^TB,\delta]$.}
\end{quote}
\item We apply (\#) to $G=C$, thus {obtaining the}  $m\times 2p^2$ matrix $H[C,\delta]$ such that
\begin{quote}
(B) {\sl The $\pi_\delta$-norms of columns in $H[C,\delta]$ do not exceed 1, and  whenever $(1-\epsilon)$-admissible contrast matrix $H$ contains  $H[C,\delta]$ as a submatrix,  $x\in\cX^s$, $\xi_x\in \Xi[x,H]$, and for some $\ell\leq p$ one of the cases $(\ell,\pm)$ takes place, $\zeta=\wh x_H(Ax+\xi_x)-x$ satisfies $\abs[C\zeta]\leq\rho^{(\ell)}[C,\delta]$.}
\end{quote}
\end{enumerate}
Given $\delta\in(0,1)$, let us set $\ov H[\delta]=[H[F^TB,\delta],H[C,\delta]]\in \bR^{m\times 4p^2}$. The following result is an immediate consequence of  (A),  (B), and (\ref{eq23}).
\begin{proposition}\label{bxinf3}
Let $B=FC$, and let us assume that a $(1-\epsilon)$-admissible  contrast matrix $H$ contains $\ov H[\delta]$ with some $\delta\in(0,1)$  as a submatrix.
Then the estimate $\wh x=\wh x_{H}(Ax+\xi_x)$ is well defined for all $x\in\cX^s$ and  $\xi_x\in\Xi[x,\overline{H}]$, and when one of the cases $(\ell,\pm)$ occurs for some
$\ell\leq p$, vector $\zeta=\wh x_H(Ax+\xi_x)-x$ obeys
\begin{align*}
\abs[F^TB\zeta]&\leq \rho^{(\ell)}[F^TB,\delta],\quad
\|C\zeta\|_{s,1}&\leq \left\|\rho^{(\ell)}[C,\delta]\right\|_{s,1},\quad
\|C\zeta\|_1&\leq 2\|C\zeta\|_{s,1}\leq 2\left\|\rho^{(\ell)}[C,\delta]\right\|_{s,1},
\end{align*}
see \rf{rhos.b}.
Hence,  the estimate $\wh B_{H}(Ax+\xi)=B\wh x_{H}(Ax+\xi_x)$  of $Bx$ {is well defined and}
 \[
\|\wh B_{H}(\omega)-Bx\|^2_2\leq 2 \rho_\ell[F^TB,\delta]\left\|\rho^{(\ell)}[C,\delta]\right\|_{s,1}.
 \]
Consequently, one has
$$
\Risk^2_{\epsilon,\|\cdot\|_2}[\wh{B}_{H}|\cX^s]\leq 2\max_{\ell\leq  p}\rho_\ell[F^TB,\delta]\left\|\rho^{(\ell)}[C,\delta]\right\|_{s,1}.
$$
\end{proposition}
Note that when $\delta\leq\epsilon/(4p^2m)$, the matrix $\ov H[\delta]$ is $(1-\epsilon)$-admissible, so that the premise of Proposition \ref{bxinf3}  holds true for $H=\ov H[\delta]$.
\subsection{Reducing complexity of the estimate}\label{sec:rce}
Let us consider again the problem of recovering the sparse signal $y=Cx$. As we have seen (Observation \ref{simpleobs}), assuming that $H$ is a $(1-\epsilon)$-admissible $m\times M$ contrast matrix, $x\in\cX^s$ and $\xi_x\in\Xi[x,H]$,  the error $\zeta=\wh x_H(Ax+\xi_x)-x$  satisfies  the inclusion
\begin{align*}
\zeta\in \mZ[H]&=\left\{ z\in \ov{\cX}:\;\|H^TAz\|_\infty\leq 2,\; \|Cz\|_1\leq 2s\|Cz\|_\infty\right\}\nn
&=\bigcup\limits_{\ell\leq p,\chi\in\{-,+\}}
 \left[\cZ_\ell^\chi[H]=\{z\in\cZ_\ell^\chi:\,\|H^TAz\|_\infty\leq2\}\right].
\end{align*}
We conclude that, under the above assumption, the error $[C\zeta]_j$ of estimating  the $j$th entry $[Cx]_j$ of $Cx$ by $C\wh{x}_H(Ax+\xi_x)$ is bounded by the quantity $\max_{z\in \mZ[H]} \left|[Cz]_j\right|$, while, by the structure of $\cZ_\ell^\pm[H]$,
\begin{align*}
&\max_{j\leq p} \max_{z\in \mZ[H]} |[Cz]_j|=\max_{j\leq p} \max_\ell\left[\max_{z\in \cZ_\ell^-[H]}[Cz]_j,\;
\max_{z\in \cZ_\ell^+[H]}[Cz]_j,\right]\\
&=\max_\ell\max\left[\max_{z\in\cZ_\ell^-[H]}\max_{j\leq p}[Cz]_j,\max_{z\in\cZ_\ell^+[H]}\max_{j\leq p}[Cz]_j\right]\\
&=\max_\ell\max\left[\max_{z\in\cZ_\ell^+[H]}\max_{j\leq p}[-Cz]_j,\max_{z\in\cZ_\ell^+[H]}\max_{j\leq p}[Cz]_j\right]\quad \hbox{[as $\cZ_\ell^-[H]=-\cZ_\ell^+[H]]$}\\
&=\max_\ell\max_{z\in\cZ_\ell^+[H]}[Cz]_\ell\quad\hbox{[by definition of $\cZ_\ell^+[H]$]},
\end{align*}
and therefore
\be
\|C\zeta\|_\infty \leq \max_{\ell\leq p}\left[r_\ell:=\max_{z} \left\{[Cz]_\ell:\,z\in \cZ^+_\ell,\,\|H^TAz\|_\infty\leq 2\right\}\right].
\ee{bred}
The above observation allows us to reduce significantly---from $2p^2$ to $p$---the row dimension of the contrast matrix $\ov H[\delta]$ without degrading the risk in recovering  $Cx$ by  $C\wh x_H$.
Given  $\delta\in(0,1)$, consider optimization problems
\be
\varrho_{\ell}^{C,\delta}=\min_{f\in \bR^m }\left[\psi_{\ell}(f):=\max_{z}\left\{(c_\ell-A^Tf)^Tz+{2}\pi_\delta(f):\,z\in \cZ_\ell^+\right\}\right],\quad  \ell=1,...,p.
\ee{simple_p1i}
We
compute optimal solutions $f_\ell$ to  problems \rf{simple_p1i}, define $h_\ell$ such that $\pi_\delta(h_\ell)=1$ and
 that $\pi_\delta(f_\ell)h_\ell=f_\ell$, $\ell\leq p$, and define $H[C,\delta]$ as the matrix with the columns $h_\ell$, $\ell\leq p$.
 Invoking Observation \ref{obsenough} and (\ref{aboutC}), we arrive at

\begin{proposition}\label{prop:reduce1}
Suppose that a $(1-\epsilon)$-admissible contrast matrix $H$ contains as a submatrix the matrix  $H[C,\delta]$ with some $\delta\in(0,1)$. The for all
$x\in\cX^s$ and $\xi_x\in\Xi[x,H]$ the recovery error $\zeta=\wh x_H(Ax+\xi_x)-x$ satisfies
\begin{equation}\label{eq88}
\|C\zeta\|_\infty\leq \varrho[C,\delta]:=\max_{\ell\leq p} \varrho^{C,\delta}_{\ell},\quad
\|C\zeta\|_1\leq2\|C\zeta\|_{s,1}\leq 2s\varrho[C,\delta],
\end{equation}
so that the risk of the recovery $\widehat{C}_H=C\widehat{x}_H$ of $Cx$ satisfies
\[
\Risk_{\epsilon,\|\cdot\|_\theta}[\wh{C}_{H}|\cX^s]\leq (2s)^{1/\theta}\varrho[C,\delta],\quad 1\leq\theta\leq\infty.
\]
\end{proposition}
{\bf Proof.} Let $x\in \cX^s$ and $\xi\in\Xi[x,H]$.  Due to \rf{bred} one has
\be
\|C\zeta\|_\infty &\leq \max_{\ell\leq p}\max_{z} \left\{[Cz]_\ell:\,z\in \cZ^+_\ell,\,\|H^TAz\|_\infty\leq 2\right\}
\leq \max_{\ell\leq p}\ov \mr_\ell
\ee{oprii}
where
\[
\ov \mr_\ell=\max_{z} \left\{[Cz]_\ell:\,z\in \cZ^+_\ell,\,\|H[C,\delta]^TAz\|_\infty\leq 2\right\}.
\]
Applying Observation \ref{sk1} with $\cW=\cZ_\ell^+$, $e=c_\ell$, $B=A$ and $\kappa=2$, we get $\ov \mr_\ell
\leq\varrho^{C,\delta}_{\ell}$, which combines with (\ref{oprii}) to prove the first inequality in (\ref{eq88}). The second inequality in (\ref{eq88}) is given by the first one combined with (\ref{aboutC}). It remains to recall that matrix $H$ by assumption  is $(1-\epsilon)$-admissible  and to invoke (!).
\qed

Note that when $\delta\leq\epsilon/p$, the matrix $H[C,\delta]$ is $(1-\epsilon)$-admissible, so that the premise in Proposition \ref{prop:reduce1} is satisfied by  $H=H[C,\delta]$.\par
Proposition \ref{prop:reduce1}  implies that when a $(1-\epsilon)$-admissible contrast matrix $H$ contains $H[C,\delta]$ with some $\delta>0$ as a submatrix, the convex compact symmetric w.r.t. the origin set
\[
\cZ[C,\delta]=\{z\in \ov\cX:\,\|Cz\|_\infty\leq \varrho[C,\delta],\,\|Cz\|_{1}\leq 2s\varrho[C,\delta]\}
\]
is an ``error localizer": whenever $x\in\cX^s$ and $\xi_x\in\Xi[x,H]$, the error $\zeta=\wh x_H(Ax+\xi_x)-x$ of recovering $x$ by $\wh x_H$ belongs to $\cZ[C,\delta]$.
  This information can be used to build an estimate of $Gx\in\bR^J$, for a given $G=[g_1,...,g_J]^T$, namely,  as follows. Given $\delta\in(0,1)$,  consider
optimization problems
\begin{equation}\label{opti3}
{\varsigma_j}[G,\delta]=\min_f\left\{\psi_{j}(f)=\max_{z}(g_j-A^Tf)^Tz+{2}\pi_\delta(f):\,z\in \cZ[C,\delta]\right\},\quad j=1,...,J.
\end{equation}
Denoting $f_j$ the optimal solutions to the problems and specifying $\pi_\delta$-unit vectors $h_\ell$ such that $f_\ell=\pi_\delta(f_\ell)h_\ell$,  let us set $\ov H[G,\delta]=[h_1,...,h_J]$. By Observation \ref{sk1} one has
$$
\max_z\left\{g_j^Tz:z\in\cZ[C,\delta],\, \|\ov H[G,\delta]^TAz\|_\infty\leq 2\right\}\leq {\varsigma_j}[G,\delta],\quad j\leq J.
$$
Taking into account that $\cZ[C,\delta]$ is symmetric w.r.t. the origin, we conclude that
\begin{equation}\label{eq83}
\max_z\left\{|g_j^Tz|:z\in\cZ[C,\delta],\, \|\ov H[G,\delta]^TAz\|_\infty\leq 2\right\}\leq {\varsigma_j}[G,\delta],\quad j\leq J.
\end{equation}
Given $J\times n$ matrix $G$ and $\delta\in(0,1)$, let $\wt{H}[\delta]=\left[H[C,\delta],\,\ov H[G,\delta]]\right]\in\bR^{m\times(p+J)}$.
\begin{proposition}\label{prop:Gred}
Assume that  a $(1-\epsilon)$-admissible contrast matrix $H$ contains as a submatrix the matrix $\wt{H}[\delta]$ for some $\delta\in(0,1)$. Whenever $x\in\cX^s$ and $\xi_x\in\Xi[x,H]$, for the error $\zeta=\wh x_H(Ax+\xi_x)-x$ of recovering $x$ by $\wh x_H$ it holds
\be
\zeta\in\cZ[C,\delta]
\ee{satisfiesmateec}
and thus
\begin{subequations}\label{satisfiesmatee}
\begin{align}
\|C\zeta\|_\theta&\leq(2s)^{1/\theta}\varrho[C,\delta],
\quad \theta\in[1,\infty],\label{satisfiesmatee.b}\\
\abs[G\zeta]&\leq{\wt{\varsigma}\,}[G,\delta],\quad {\wt{\varsigma}\,}[G,\delta]:=\left[{\varsigma}_{1}[G,\delta];...;{\varsigma}_{J}[G,\delta]\right].\label{satisfiesmatee.e}
\end{align}
\end{subequations}
In particular,
\[
\|G\zeta\|_\infty\leq \max_{j\leq J}{\varsigma}_j[G,\delta]=:{\ov\varsigma}[G,\delta],
\]
so that the $(s,\epsilon,\|\cdot\|_\infty)$-risk of the recovery $\wh G_{ H}(\cdot)=G\wh x_{H}(\cdot)$ of the image $Gx$ of the unknown signal $x$ satisfies
\be
\Risk_{\epsilon,\|\cdot\|_\infty}[\wh{G}_{H}|\cX^s]\leq {\ov\varsigma}[G,\delta].
\ee{risk666}
Furthermore, let, in addition to what has been already assumed,  $B\in \bR^{\nu\times n}$ be such that $B=FC$ for some $F\in \bR^{\nu\times p}$, and let $G=F^TB$.
Then the $(s,\epsilon,\|\cdot\|_2)$-risk of the estimate $\wh B_{ H}(\cdot)=B\wh x_H(\cdot)$
of $Bx$ satisfies
\[
\Risk^2_{\epsilon,\|\cdot\|_2}[\wh B_{ H}|\cX^s]\leq 2s\varrho[C,\delta],{\ov\varsigma}\left[{F^TB},\delta\right].
\]
\end{proposition}
{\bf Proof.} The validity of (\ref{satisfiesmateec}) has been justified when introducing the set $\cZ[C,\delta]$ (recall that we are under the assumptions that $\xi_x\in\Xi[x,H]$ and that $H[C,\delta]$ is a submatrix of $H$). Thus, $\zeta\in\cZ[C,\delta]$, and (\ref{satisfiesmatee.b}) holds true  due to (\ref{eq14}), see Observation \ref{simpleobs}.
Inequalities \rf{satisfiesmatee.e} are given by (\ref{eq83}) (recall that we are under the assumption that $\ov H[G,\delta]$ is a submatrix of $H$), and(\ref{risk666}) is given by (\ref{satisfiesmatee.e}) combined with $(!)$. \par
Finally, when $B=FC$ and $G=F^TB$ we have
\[
\|B\zeta\|^2_2=(C\zeta)^TF^TB\zeta\leq \| C\zeta\|_1\|F^TB\zeta\|_\infty\leq 2s\|C\zeta\|_\infty\|F^TB\zeta\|_\infty\leq2s \varrho[C,\delta]\,\ov\rho[F^TB,\delta];
\]
invoking (!), we arrive at the concluding risk bound. \qed
\par
Note that when $\delta\leq\epsilon/(p+J)$, matrix $\wt{H}[\delta]$ is $(1-\epsilon)$-admissible, and the premise of Proposition \ref{prop:Gred} is satisfied by $H=\wt{H}[\delta]$.

Next, let $G=C$, and let us consider the ``iterated contrast"
$\wt{H}[C,\delta]=\left[H[C,\delta],\ov H[C,\delta]\right]\in\bR^{m\times 2p}$. Applying Proposition \ref{prop:Gred} in this setting we arrive at the following immediate
\begin{observation}\label{prop:Cred}
Suppose that a $(1-\epsilon)$-admissible contrast matrix $H$ contains as a submatrix the matrix  $\wt H[C,\delta]$ with some $\delta\in(0,1)$. Then for all
$x\in\cX^s$ and $\xi_x\in\Xi[x,H]$ the error $\zeta=\wh x_H(Ax+\xi_x)-x$ of recovering $x$ by $\wh x$ satisfies
\begin{subequations}\label{satisfiesmateeC}
\begin{align}
\|C\zeta\|_\theta&\leq(2s)^{1/\theta}\varrho[C,\delta],\quad 1\leq\theta\leq\infty.
\label{satisfiesmateecC1}\\
\abs[C\zeta]&\leq{\wt{\varsigma}\:}
[C,\delta] \label{satisfiesmateeC.c}
\end{align}
\end{subequations}
(for notation, see \rf{satisfiesmatee.e}),
and thus
\be
\|C\zeta\|_\theta&\leq2^{1/\theta}\left\|{\wt{\varsigma}\;}[C,\delta]\right\|_{s,\theta},\quad 1\leq\theta\leq\infty.
\ee{satisfiesmateeC.d}
As a result, the risk of the recovery $\widehat{C}_H=C\widehat{x}_H$ of $Cx$ satisfies
\[
\Risk_{\epsilon,\|\cdot\|_\theta}[\wh{C}_{H}|\cX^s]\leq
\min\left[(2s)^{1/\theta}\varrho[C,\delta], \,2^{1/\theta}\|\rho[C,\delta]\|_{s,\theta}.\right],\quad 1\leq\theta\leq\infty.
\]
\end{observation}
{\bf Proof.}  \rf{satisfiesmateecC1} and  \rf{satisfiesmateeC.c} are readily given by \rf{satisfiesmatee}. Now \rf{satisfiesmateeC.d} is implied by \rf{satisfiesmateeC.c} along with
$\|C\zeta\|_\theta\leq 2^{1/\theta}\|C\zeta\|_{s,\theta}$
due to Lemma \ref{simple1}.
  The final risk bound is given by \rf{satisfiesmateeC.d}  combined with $(!)$.
\qed

\section{Recovering  linear image of a sparse signal: an alternative approach}\label{Sec2b}
In this section we present an alternative approach to bounding of the recovery error and design of contrast matrix $H$.
We make the following standing assumptions:
\begin{enumerate}
\item {\em The symmetrization $\ov\cX=\cX-\cX$ of the signal set belongs to the basic ellitope}\footnote{Ellitopes, as defined in \cite[Section 4.2.1]{PUP}, are compact convex symmetric w.r.t. the origin sets delimited by convex quadratic surfaces. Examples include bounded intersections of finitely many ellipsoids/elliptic cylinders centered at the origin and $\|\cdot\|_p$-balls with $2\leq p\leq\infty$.}  $\wt\cX$:
\be
\begin{array}{c}
\wt{\cX}=\{x\in\bR^n:\exists t\in\cT:x^TT_kx\leq t_k,\,k\leq K\}\\
\end{array}
\ee{ellit}
{\em where $T_k\succeq0$, $\sum_k T_k\succ0$, and $\cT\subset\bR^n_+$  is a monotone (i.e., $t\in\cT\ \&\ 0\leq t'\leq t \Rightarrow t'\in \cT$)  convex compact subset of $\bR^K_+$ intersecting ${\bR^K_{++}=\inter\bR^K_+}$.}
\item {\sl The unit ball $\cB_*$ of the norm $\|\cdot\|_*$ conjugate to the norm $\|\cdot\|$ used to measure the recovery error is a basic ellitope}, namely,
\begin{equation}\label{Bstar}
\cB_*:=\{u:u^Tw\leq 1\,\forall (w,\|w\|\leq1)\}=\{u:\exists r\in\cR: \,w^TR_j w\leq r_j,\,j\leq J\}
\end{equation}
where $R_\ell\succeq0 $ and $\sum_\ell R_\ell\succ0$, and $\cR\subset\bR^\nu_+$ is a monotone convex compact set intersecting $\bR^\nu_{++}$.
\end{enumerate}
\subsection{Bounding the recovery error} \label{sec:41}
\subsubsection{Initial  risk bound}
Note that a point  $z\in \mZ[H]$, see  (\ref{thatis}), can be augmented with $\eta\in \bR^p$ to belong to the set $\mY[H]=\bigcup\limits_{\ell=1}^p\mY^\ell[H]\subset \bR^n\times \bR^p$, with

\[
{\mY_\ell[H]}=\left\{[z;\eta]\in\bR^{n+p}:\,z\in \ov\cX,\,\|H^TAz\|_\infty\leq 2,\,[z;\eta]\in\cK_\ell\right\}
\]
where
\begin{subequations}\label{cKell}
\begin{equation*}
\cK_{\ell}=\left\{[z;\eta]\in\bR^{n+p}:\;\;
\begin{array}{ll}
\eta_i-[Cz]_i\geq 0,\,\eta_i+[Cz]_i\geq 0,\,\eta_i- \eta_{\ell}\leq0,\,i\leq p,\quad
  & \refstepcounter{equation}(\theequation)\manuallabel{cKell-a}{\theequation}\\
\sum_{i\leq p}\eta_i- 2s \eta_{\ell}\leq0,
  & \refstepcounter{equation}(\theequation)\manuallabel{cKell-b}{\theequation}\\
{[Cz]_i^2-\eta_i^2}\leq0,\,\eta_i^2-[Cz]_i^2\leq 0,\,i\leq p,
  & \refstepcounter{equation}(\theequation)\manuallabel{cKell-c}{\theequation}\\
\sum_{i\leq p}\eta_i^2- 2s \eta_{\ell}^2\leq0
  & \refstepcounter{equation}(\theequation)\manuallabel{cKell-d}{\theequation}
\end{array}
\right\}
\end{equation*}
\end{subequations}

(indeed, it suffices to set $\eta_i=|[Cz]_i|,\,i\leq p$,  and notice that for $z\in \mZ[H]$ one has $\|Cz\|_1\leq2s\|\cZ\|_\infty$ and  $\|Cz\|^2_2\leq 2s\|Cz\|_\infty^2$, cf. Lemma \ref{simple1}).\par
As a result, for $\zeta\in \mZ[H]$, one has
\[
\|B\zeta\|\leq \max_{\ell\leq p} \Opt_\ell[H],\,\,\Opt_\ell[H]=\max_{[z;\eta]\in \mY_\ell[H]}\|Bz\|.
\]
Note that  the sets $\mY_\ell[H]$ {satisfy {\Antimonotonicity} condition}: whenever $H'$ is a submatrix of $H$, $\mY_\ell[H]$ is a subset   of $\mY_\ell[H']$ and, consequently, $\Opt_\ell[H]\leq\Opt_\ell[H']$.

Assuming $H$ to be $(1-\epsilon)$-admissible, for $x\in \cX^s$ and $\xi_x\in\Xi[x,H]$ the recovery error $\zeta=\wh x_H(Ax+\xi)-x$ belongs to $\mZ[H] $
(Observation \ref{simpleobs}); when recalling that the event $\xi_x\in\Xi[x,H]$  has $P_x$-probability $\geq1-\epsilon$, we conclude that
\begin{equation}\label{premain}
\Risk_{\epsilon,\|\cdot\|}[\widehat{B}_H|\cX^s]\leq \Opt[H]=\max_{\ell\leq p} \Opt_\ell[H].\quad[\wh B_H(\cdot)=B\wh x_H(\cdot)]
\end{equation}
\subsubsection{Efficiently computable relaxation of (\ref{premain})} \label{5.1.2} The quantities $\Opt_\ell[H]$ appear difficult to compute, which makes the above bound intractable to evaluate -- let alone optimize over $H$. We circumvent this difficulty to some extent by replacing the computationally challenging quantities $\Opt_\ell[H]$ with their efficiently computable upper bounds. To upper-bound $\Opt_\ell[H]$, we act as follows

\begin{enumerate}
\item We construct a finite collection  $\cS_\ell$ of symmetric matrices $S_{\imath\ell}$, $\imath\leq I_\ell=I$, such that for every $[z;\eta]\in\cK_\ell$ it holds $[z;\eta]^TS_{\imath\ell} [z;\eta]\leq0$.
Specifically, we start the collection with the $2p+1$ matrices of the quadratic forms in the left-hand sides of inequalities \rf{cKell-c} and \rf{cKell-d}.
Then we rewrite the system of homogeneous linear inequalities \rf{cKell-a} and \rf{cKell-b} as $f_{\ell i}^T[z;\eta]\geq0$, $i\leq J_\ell$, and augment  the collection with the  matrices of quadratic forms $-[z;\eta]^T[f_{\ell i}f_{\ell j}^T+ f_{\ell j}f_{\ell i}^T][z;\eta]$, $1\leq i<j\leq J_\ell$.
\item
Now, let $\overline{B}=[B,0_{\nu\times p}]$, and let for a vector $v$, $\Diag\{v\}$  stand for the diagonal matrix with entries of $v$ on the diagonal. Denote
$$
\phi_Z(y)=\sup_{z\in Z}y^Tz
$$
the support function of a set $Z$.  We need the following
\begin{lemma}\label{prop:2rbound}
In the situation of this section, given a $(1-\epsilon$)-admissible $H\in \bR^{m\times M}$, for $\ell=1,...,p$, consider the optimization problems
\begin{align}\label{maineq}
\ov\Opt_\ell[H]&=\min_{\rho,\gamma,\upsilon,\lambda}\Big\{\phi_{\cR}(\rho)+\phi_{\cT}(\gamma)+4\sum_i\upsilon_i:\;\rho\geq 0,\upsilon\geq 0,\,\gamma\geq 0,\,\lambda\geq 0,\\
&\qquad  \left.\left[\begin{array}{c|c}\sum_j\rho_jR_j&{1\over 2}\overline{B}\cr\hline
{1\over 2}\overline{B}^T&
   \left[\begin{array}{c|c}{A}^TH\Diag\{\upsilon\}H^T{A}+\sum_k\gamma_kT_k&0_{n\times p}\cr\hline0_{p\times n}&0_{p\times n}
   \end{array}\right]+\sum_\imath \lambda_\imath {S_{\imath\ell}}\\
   \cr\end{array}\right]\succeq 0.\right\}\nonumber
\end{align}
Then $\Opt_\ell[H]\leq \ov\Opt_\ell[H]$ for all $\ell\leq p$. As a result,
\par{\rm (i)} whenever $x\in\cX^s$, $\xi_x\in\Xi[x,H]$, and for some $\ell\leq p$ one of the cases $(\ell,\pm)$ occurs,
the estimation error $B\zeta$, $\zeta=\wh x_H(Ax+\xi_x)-x$, of the polyhedral estimate $\wh B_H=B\wh x_H$ of $Bx$ satisfies
\[
\|B\zeta\|\leq\ov\Opt_\ell[H];
\]
\par{\rm(ii)} one has
\be
\Risk_{\epsilon,\|\cdot\|}[\widehat{B}_H|\cX^s]\leq \ov\Opt[H]=\max_{\ell\leq p} \ov\Opt_\ell[H].
\ee{l2rb}
\end{lemma}
{\bf Proof.} Let us fix $\ell\leq p$, and let
 $(\rho,\gamma,\upsilon,\lambda)$ be a feasible solution to the $\ell$-th optimization problem in (\ref{maineq}).
Let also $[\zeta;\eta]\in\mY_\ell[H]$ and $u\in \cB_*$.  One has
\begin{equation}\label{chain}
\begin{array}{rcl}
u^TB\zeta&=&u^T\overline{B}[\zeta;\eta]\\
&\leq& \sum_j\rho_ju^TR_ju +\zeta^TAH\Diag\{\upsilon\}H^TA\zeta +\sum_k\gamma_k\zeta^TT_k\zeta+\sum_\imath\lambda_\imath[\zeta;\eta]^TS_{\imath\ell}[\zeta;\eta]\\
&& \hbox{[by the semidefinite constraint in (\ref{maineq})]}\\
&
\leq&\sum_j\rho_ju^TR_ju+4\sum_j\upsilon_j+\sum_k\gamma_k\zeta^TT_k\zeta
\end{array}
\end{equation}
where the last inequality is a consequence of $\|H^TA\zeta\|_\infty\leq 2$ and $[\zeta;\eta]^TS_{\imath\ell}[\zeta;\eta]\leq0$ due to
$[\zeta;\eta]\in\mY_\ell[H]\subset\cK_\ell$.
As $u\in \cB_*$, there exists $r\in\cR$ such that $u^TR_ju\leq r_j$, $j\leq J$, whence \[\sum_j\rho_ju^TR_ju\leq \sum_j\rho_jr_j\leq\phi_\cR(\rho)\]
As  $[\zeta;\eta]\in\mY_\ell[H]$, we have $\zeta\in \overline{\cX}\subset\wt\cX$; hence, for the same reasons as above,  $\sum_k\gamma_k\zeta^TT_k\zeta\leq\phi_{\cT}(\gamma)$. Thus, whenever $[\zeta;\eta]\in \mY_\ell[H]$, it holds
$$
\|B\zeta\|=\max_{u\in\cB_*}u^TB\zeta\leq 4\sum_{j=1}^m\upsilon_j+\phi_\cR(\rho)+\phi_\cT(\gamma).
$$
This relation holds for every feasible solution to the $\ell$-th problem in (\ref{maineq}), implying that $\|B\zeta\|\leq \ov\Opt_\ell[H]$
for every $[\zeta;\eta]\in\mY_\ell[H]$, and therefore $\Opt_\ell[H]\leq \ov\Opt_\ell[H]$, as claimed.
\par
It remains to verify (i) and (ii). (ii) is an immediate consequence of (i) due to (!) and (!!). To verify (i), note that when $x\in\cX^s$, $\xi_x\in\Xi[x,H]$, and for some $\ell$
one of the cases $(\ell,\pm)$ occurs, in the notation of (i) we have $\zeta\in\mZ[H]$ by Observation \ref{simpleobs}  and $\|C\zeta\|_\infty=|[C\zeta]_\ell|$ (recall how cases are defined). It follows that setting $\eta_i=|[C\zeta]_i|$, $i\leq p$, we ensure $[\zeta;\eta]\in \mY_\ell[H]$, so that $\|B\zeta\|\leq\Opt_\ell[H]\leq\ov\Opt_\ell[H]$. \qed

\end{enumerate}
\subsubsection{Designing the estimate}\label{secdesalt}
Our next objective is to optimize risk bound \rf{l2rb} over the contrast matrix $H$.
\subsubsection{The case of sub-Gaussian/Gaussian observation scheme}\label{5.1.4}
Assume that the observation scheme is Gaussian or sub-Gaussian with noise intensity $\sigma$. Here we act as follows.
\par
 Given $\delta\in(0,1)$, let us set
    \begin{equation}\label{Gaussianpi}
    \pi_\delta(\cdot)=\mn\|\cdot\|_2,\,\,\mn\big[=\mn[\delta]\big]:=\left\{
    \begin{array}{ll}\sigma\chi_{\delta},&\mbox{case of Gaussian observation scheme,}\\
    \sigma\sqrt{2\ln(2/\delta)},&\mbox{case of {sub-}Gaussian observation scheme,}\\
    \end{array}
    \right.
\end{equation}
    where, same as above, $\chi_\alpha$ is $(1-\alpha/2)$-quantile of $\cN(0,1)$; note that $\pi_\delta$ is the  norm
    associated with the observation scheme in question, see Section \ref{sec:obss}.

 Consider   $p$ convex optimization problems
 \begin{equation}\label{problems}
 \begin{array}{rcl}
\ov\Opt_\ell[\delta]&=&\min_{\rho.\gamma,\upsilon,\lambda,\Theta}\Big\{\phi_\cR(\rho)+\phi_\cT(\gamma)+4\mn^2[\delta]\Tr(\Theta):\;\upsilon\geq 0,\,\gamma\geq 0,\,\lambda\geq 0,\;\Theta\succeq 0\\
&&\qquad \left.\left[\begin{array}{c|c}\sum_j\rho_jR_j&{1\over 2}\overline{B}\cr\hline
{1\over 2}\overline{B}^T&
   \left[\begin{array}{c|c}{A}^T\Theta{A}+\sum_k\gamma_kT_k&0_{n\times p}\cr\hline0_{p\times n}&0_{p\times n}
   \end{array}\right]+\sum_\imath \lambda_\imath S_{\imath\ell}\cr\end{array}\right]\succeq0
\right\}\\
\end{array}
\end{equation}
Given an optimal solution $(\ov{\rho}^\ell,\,\ov{\gamma}^\ell,\,\ov{\lambda}^\ell,\,\ov{\Theta}_\ell)$ to the $\ell$-th problem, we subject
$\ov{\Theta}_\ell$ to the eigenvalue decomposition: $\ov{\Theta}_\ell =\sum_{i=1}\rho_{i\ell}f_{i\ell}f_{i\ell}^T$ with the eigenvectors $f_{i\ell}$ normalized to have unit Euclidean norms.
We then set
    $$
    H_\ell[\delta]=\mn^{-1}[f_{1\ell},...,f_{m\ell}],\,\,\overline{\upsilon}^\ell_i=\mn^2\rho_{i\ell},
    $$
This ensures that the columns of $H_\ell[\delta]$ have unit $\pi_\delta$-norms  and that
\begin{equation}\label{soso}
\ov{\Theta}_\ell= H_\ell[\delta]\Diag\{\ov{\upsilon}^\ell\}H^T_\ell[\delta],
\end{equation}
hence
\begin{equation}\label{sososo}
\ov{\Opt}_\ell[\delta]=\phi_\cR(\ov{\rho}^\ell)+\phi_{\cT}(\ov{\gamma}^\ell)+
4\mn^{2}{\Tr(\ov{\Theta}_\ell)}=\phi_\cR(\ov{\rho}^\ell)+\phi_\cT(\ov\gamma^\ell)+4\sum_i\ov{\upsilon}^\ell_i.
\end{equation}
Let us make the following
\begin{observation}\label{obsnew} Given $\delta\in(0,1)$, the matrix  $H_\ell[\delta]$ just defined has columns of unit  $\pi_\delta$-norm and
possesses the following property:
whenever $H$ is a $(1-\epsilon)$-admissible contrast matrix containing $H_\ell[\delta]$ as a submatrix,
one has
\begin{equation}\label{eeee}
\ov\Opt_\ell[H]\leq\ov\Opt_\ell[\delta]
\end{equation}
with $\ov\Opt_\ell[H]$ given by (\ref{maineq}).
\end{observation}
Indeed, under the premise of Observation, relations (\ref{soso}) and (\ref{sososo}) show that $(\ov{\rho}^\ell,\ov{\upsilon}^\ell,\,\ov{\gamma}^\ell,\,\ov{\lambda}^\ell)$
is a feasible solution to the $\ell$-th optimization problem in (\ref{maineq}) with $H$ set to $H_\ell[\delta]$, the value of the objective of the latter problem at this solution  being $\ov\Opt_\ell[\delta]$. Consequently, $\ov\Opt_\ell[H_\ell[\delta]] \leq \ov\Opt_\ell[\delta]$. It remains to note that by \Antimonotonicity, $\ov\Opt_\ell[H]\leq \Opt_\ell[H_\ell[\delta]]$ whenever $H_\ell[\delta]$ is a submatrix of $H$. \qed
\par
As an immediate consequence of  Observation \ref{obsnew}, when setting $M=pm$, $ \delta=\ov\delta:=\epsilon/M$, and $H=H_*[\ov\delta]:=[H_1[\ov\delta],...,H_p[\ov\delta]]$
(under the circumstances, this is our recommended contrast design), we obtain a $(1-\epsilon)$-admissible contrast matrix $H$ such that $\ov\Opt_\ell[H]\leq \ov\Opt_\ell[\ov\delta]$ for all $\ell$, implying by Lemma  \ref{prop:2rbound} that
\begin{equation}
\Risk_{\epsilon,\|\cdot\|}[\widehat{B}_H|\cX^s]\leq \max_{\ell\leq p} \ov\Opt_\ell[\ov\delta].
\end{equation}
{\bf Remark.} In fact the contrast design we have presented minimizes the risk bound (\ref{l2rb}) over contrast matrices $H$ with
$\pi_{\overline{\delta}}$-norms of columns $\leq 1$. Specifically, for every $m\times M$ matrix $H$ of this type and every $\ell$ it holds $\ov\Opt_\ell[H]\geq\ov\Opt_\ell[\ov\delta]$.
\begin{quote}
Indeed, let  $\ell\leq p$ and let $(\rho,\gamma,\upsilon,\lambda)$ be a feasible solution to the $\ell$-th problem in \rf{maineq}.
Setting $\Theta = H\Diag\{\upsilon\}H^T$, we have $\Theta\succeq0$ and
\[\Tr(\Theta)=\sum_{j=1}^M\upsilon_j\|\Col_j[H]\|_2^2= \mn^{-2}[\ov\delta]\sum_j\upsilon_j\pi^2_{\ov\delta} (\Col_j[H])\leq \mn^{-2}[\ov\delta]\sum_j\upsilon_j,
\] that is,
\[4\sum_j\upsilon_j\geq4\mn^2[\ov\delta]\Tr(\Theta).\]
Comparing optimization problems in (\ref{maineq}) with those in (\ref{problems}), we see that when $\delta=\ov\delta$, a feasible solution $(\rho,\gamma,\lambda,\upsilon)$ to the $\ell$-th problem of the first family induces
a feasible solution $(\rho,\gamma,\lambda,\Theta)$ to the $\ell$-th problem of the second one with  $\delta=\ov\delta$. Furthermore, the value of the
objective of the former problem at its feasible solution $(\rho,\gamma,\lambda,\upsilon)$ is $\geq$ the value of the objective of the
latter problem at its feasible solution $(\rho,\gamma,\lambda,\Theta)$ and thus is $\geq\ov\Opt_\ell[\ov\delta]$, implying that
$\ov\Opt_\ell[H]\geq\ov\Opt_\ell[\ov\delta]$, as claimed.
\end{quote}
As we  have just seen, for every $H$ with $\pi_{\ov\delta}$-norms of columns $\leq1$ it holds $\ov\Opt_\ell[H]\geq\ov\Opt[\ov\delta]$; as $H_*=H_*[\ov\delta]$ belongs to this family, let it be called $\cH[\ov\delta]$, and ensures, by Observation \ref{obsnew}, that $\ov\Opt_\ell[H_*]\leq\ov\Opt_\ell[\ov\delta]$ for all $\ell$, $H_*$ indeed optimizes the right hand side of  (\ref{l2rb}) over the contrast matrices from $\cH[\ov\delta]$.
\subsubsection{Poisson and Discrete observation schemes}
We have described  ``presumably optimal contrast design" in the case of sub-Gaussian and Gaussian observation schemes.
Applying the results of \cite[Section 5.1.5]{PUP}, the above constructions can be straightforwardly
modified to yield a ``presumably optimal contrast design" in the case of Discrete and Poisson observation schemes. \par
We assume that matrix $A$ is column-stochastic and $\cX$ is a subset of the probabilistic simplex in the case of Discrete
observation scheme, and suppose that $A$ is entry-wise nonnegative and $\cX$ belongs to the nonnegative orthant in the Poisson case.
A sketch of required modifications is as follows:
\begin{enumerate}
\item Given the tolerance parameter $\delta$ of the norm $\pi_{\delta}(\cdot)$  (eventually $\delta$ will be set to $\epsilon/(pm)$), we specify a convex compact set
$\cW=\cW_\delta\subset \bR^m_+$ as follows:
\begin{itemize}\item  in the case of  Discrete observation scheme with aggregated $\cK$-repeated observation \rf{aggr} we put
\[\cW= {4\ln(2/\delta)\over \cK}A\cX+{64\ln^2(2/\delta)\over 9\cK^2}\Delta_m\]
 where $\Delta_m=\left\{u\in\bR^m_+:\sum_iu_i=1\right\}$ ;
\item in the case of Poisson observation scheme with aggregated $\cK$-repeated observation we put
\[\cW={4 \ln(2/\delta)\over \cK}A\cX + {16\ln^2(2/\delta)\over 9\cK^2}\Delta_m.\]
\end{itemize}
\item We replace the terms $4\mn^{2}\Tr(\Theta_\ell)$ in the objectives of problems (\ref{problems}) with
$$
24 \ln(2\sqrt{3}m^2)\phi_\cW(\dg(\Theta_\ell)),
$$
where $\dg(Q)$ is the vector composed of the diagonal entries of a square matrix $Q$;
\item When converting $\Theta_\ell$ to $H_\ell[\delta]$,  the eigenvalue decomposition is substituted by the randomized algorithm described in
\cite[Lemma 5.6(ii)]{PUP}.
\end{enumerate}
\hide{
\paragraph{\ref{secdesalt}.4 Illustration} We present here results of a ``proof of concept"  numerical experiment, in which  $A$ is an $m\times n$ matrix drawn from the Gaussian ensemble normalized to have spectral norm
1 (our $A$ was 2-good), and $B$ is an $m\times n$ matrix selected from the similar ensemble. {\crd Also with $\|B\|_*=1$?}

We use $\cX=\{x\in\bR^n:\,\|x\|_\infty\leq1\}$; 
and Gaussian observation scheme with $\cN(0,\sigma^2I_m)$ observation noise. The results of the experiment for $s=3$ are presented in Table \ref{tablast}.
\begin{table}
$$
\begin{array}{||c||c|c|c||c|c||}
\cline{2-6}
\multicolumn{1}{c||}{}&\multicolumn{3}{c||}{\hbox{\scriptsize recovery errors}}&\multicolumn{2}{c||}{\hbox{\scriptsize risk bound}}\\
\hline
\sigma&\hbox{\scriptsize median}&\hbox{\scriptsize mean}&\hbox{\scriptsize max}&\hbox{\scriptsize sparse}
&\hbox{\scriptsize dense}\\
\hline\hline
0.1&0.1160&0.1214&0.1958&0.8457&1.4982\\
\hline
0.01&
0.0133&0.0147&0.0464&0.5090&0.9196\\
\hline
0.001&0.0014&0.0016&0.0079&0.4528&0.8432\\
\hline
\multicolumn{3}{||l|}{\hbox{\scriptsize$m=\nu=20,n=24,s=3.$}}&\multicolumn{3}{c}{}\\
\cline{1-3}
\hline
0.1&0.1024&0.1030&0.1710&0.7292&1.6372\\
\hline
0.01&
0.0134&0.0148&0.0359&0.4701&1.2251\\
\hline
0.001&0.0013&0.0015&0.0034&0.3783&1.1653\\
\hline
\multicolumn{3}{||l|}{\hbox{\scriptsize$m=\nu=24,n=32,s=3.$}}&\multicolumn{3}{c}{}\\
\cline{1-3}
\end{array}
$$
\caption{\label{tablast} Bounds on $\Risk_{0.01,\|\cdot\|_2}$ and $\|\cdot\|_2$-norms  of recovery errors, data over $N=100$ simulations with 3-sparse signals. Risk bound:
sparse---the right-hand side in \rf{l2rb} corresponding to $s=3$;  dense---risk bound of provably near-optimal polyhedral estimate from  \cite[Section 5.1.5]{PUP},  the signal set being $\cX$.}
\end{table}4
}
\subsection{Putting things together}\label{together}
Our present goal is to combine the techniques from Sections \ref{Sec2a} and \ref{5.1.4} to get polyhedral estimates obeying improved risk bounds.

Here, same as in Section \ref{Sec2b}, we suppose that the signal set  $\cX$ satisfies (\ref{ellit}).
 We also assume that the observation scheme is either the Gaussian, or the sub-Gaussian with noise intensity $\sigma$, so that $\pi_\delta$ is given by (\ref{Gaussianpi}).
\subsubsection{Preliminaries} To proceed, we need a slightly modified intermediate result  from the previous section.
Given $\delta\in(0,1)$, for
 $\ell=1,...,p$ and $U\in\bS^n$, let
\begin{equation}\label{U}
\begin{array}{rl}
\Opt_{\ell}(U,\delta)=\min_{\Theta,
\gamma,\lambda}&\Big\{\phi_\cT(\gamma)+4\mn^2[\delta]\Tr(\Theta):\Theta\succeq0,\gamma\geq0,\lambda\geq0,\\
&~~\left[\begin{array}{c|c}A^T\Theta A+\sum_k\gamma_kT_k-U&0_{n\times p}\cr\hline
 0_{p\times n}&0_{p\times p}\cr\end{array}\right]+\sum_\imath \lambda_\imath S_{\imath\ell}\succeq0\Bigg\}
 \end{array}
 \end{equation}
 with $\mn[\delta]$ given by (\ref{Gaussianpi}) and matrices $S_{\imath\ell}$ given by the construction from
 Section \ref{5.1.2}.
 Note that $\Opt_{\ell}(U,\delta)$ are efficiently computable convex real-valued functions of $U\in\bS^n$ depending on $\delta\in(0,1)$ as a parameter.\par
 We convert the $\Theta$-component $\Theta_\ell$ of an optimal solution to the $\ell$-th problem (\ref{U}) into an $m\times m$ matrix $H_\ell[U,\delta]$ with  columns of unit $\pi_\delta$-norm. Specifically, we subject $\Theta_\ell$ to eigenvalue decomposition $\Theta_\ell =G_\ell \Diag\{\mu^\ell\}G_\ell^T$ and build $H_\ell[U,\delta]$ by scaling (dividing by ${\mn[\delta]}$) the columns of the orthonormal matrix  $G_\ell$.
\begin{proposition}\label{propnew}
Let $\delta\in(0,1)$, an let $(1-\epsilon)$-admissible matrix $H$ contain $H_\ell[U,\delta]$ with some $\ell\leq p$. The estimate $\wh{x}_{H}(Ax+\xi_x)$, see (\ref{polyest}), is well defined for all  $x\in\cX^s$ and $\xi_x\in\Xi[x,H]$, and whenever one of the cases $(\ell,\pm)$ takes place, the
 recovery error $\zeta=\wh{x}_{ H}(Ax+\xi_x)-x$ satisfies
 \begin{equation}\label{somenew}
 \zeta^TU\zeta\leq \Opt_{\ell}(U,\delta).
 \end{equation}
In particular, if {$U=B^TB$ and} $H$ contains as submatrices all matrices $H_\ell[U,\delta]$, $\ell\leq p$, the estimate $\widehat{B}_H(\cdot)=B\wh x_H(\cdot)$ of $Bx$
 satisfies
 $$
 {\Risk^2}_{\epsilon,\|\cdot\|_2}[\widehat{B}_H|\cX^s]\leq\max_\ell\Opt_{\ell}(U,\delta).
 $$
 \end{proposition}
 {\bf Proof.}  Let us fix $x\in\cX$, $\xi\in\Xi[x,H]$, and $\ell\leq p$ such that one of the cases $(\ell,\pm)$ occurs. Let also
  $\zeta=\widehat{x}_H(Ax+\xi)-x$, and let $(\Theta_\ell,\gamma^\ell,\lambda^\ell)$ be an optimal solution to $\ell$-th problem in (\ref{U}). By repeating the argument of the proof of Lemma  \ref{prop:2rbound}, setting $\eta_i=|[C\zeta]_i|$, $i\leq p$, we conclude that $[\zeta;\eta]\in\mY_\ell[H]{\subset\cK_\ell}$, whence, due to the origin of the matrices $S_{\imath\ell}$, we obtain
 $$[\zeta;\eta]^TS_{\imath\ell}[\zeta;\eta]\leq 0\;\;\forall \imath.
 $$
 This combines with $\lambda^{\ell}\geq0$ and the semidefinite constraint in \rf{U} to imply that
 \begin{equation}\label{new11}
 \zeta^TU\zeta\leq \zeta^TA^T\Theta_{\ell}A\zeta+\sum_k\gamma^{\ell}_k\zeta^TT_k\zeta.
 \end{equation}
 Now, by construction, $\Theta_{\ell}=H_{\ell}[U,\delta]\Diag\{\mu^{\ell}\}H^T_{\ell}[U,\delta]$ with $\mu^{\ell}\geq0$, and  $\sum_i\mu^{\ell}_i={\mn^2[\delta]}\Tr(\Theta_{\ell})$, which combines with $\|H^TA\zeta\|_\infty\leq2$ to imply that
 $$
 \zeta^TA^T\Theta_{\ell}A\zeta\leq 4\sum_i\mu^{\ell}_i=4{\mn^2[\delta]}\Tr(\Theta_{\ell}).
 $$
 Finally, as we remember, from $\zeta\in\overline{\cX}\subset \wt\cX$ and $\gamma^{\ell}\geq0$ it follows that $\sum_k\gamma^{\ell}_k\zeta^TT_k\zeta\leq
 \phi_{\cT}(\gamma^{\ell})$. Thus, (\ref{new11}) implies (\ref{somenew}). \qed

We continue
with the following observation, cf. \cite[Observation 3.1]{juditsky2024design}.
\begin{lemma}\label{lem-obs} Let $\theta\in[1,2]$, $\vartheta=\frac{\theta}{2-\theta}$, and let  $\varsigma\in\bR^\nu$ and $S, U\in\bS^n$ be such that
\be\left[\begin{array}{c|c}S+U&B^T\cr\hline B&\Diag\{\varsigma\}\cr\end{array}\right]\succeq0,\quad
\|\varsigma\|_{\vartheta}\leq 1.
\ee{LMI}
Then for any $z\in \bR^n$,
$
\|Bz\|^2_{\theta}\leq z^TSz+z^TUz.
$\footnote{This bound is tight: the inequality becomes equality when optimizing the bound in $S, U\in \bS^n$ and $\varsigma\in\bR^\nu$ for a given $x\in \bR^n$, cf. \cite[Section A.1]{juditsky2024design}.}
\end{lemma}

\subsubsection{The construction}
Consider the following construction of the contrast matrix $H$. 
\par
Let us fix $\delta\in(0,1)$.
\begin{enumerate}
\item
For $G=[g_1,...,g_p]^T\in \bR^{p\times n}$ and $\ell,j=1,...,p$, we set (cf. (\ref{rhos}))
\begin{align}
\rho^\pm_{\ell j}[G,\delta]&=\min_f\left\{\max_{z}\left[(g_j-A^Tf)^Tz+{2}\pi_\delta(f):\,z\in \cZ_\ell^\pm\right]\right\},
\label{newpsi}\\
\rho_{\ell j}[G,\delta]&=\max\left[\rho^-_{\ell j}[G,\delta],\rho^+_{\ell j}[G,\delta]\right],\nn
\rho^{(\ell)}[G,\delta]&=\left[\rho_{\ell 1}[G,\delta];...;\rho_{\ell p}[G,\delta]\right],\nn
\rho_\ell
[G,\delta]&={\max_{j\leq p}\rho_{\ell j}[G,\delta]}.\nonumber
\end{align}
Note that $\rho^\pm_{\ell j}[G,\delta]$ are efficiently computable convex in $G$ real-valued functions.
\item Next, let
\begin{align}\label{UUu}
\Opt_{\ell}(U,\delta)=\min_{\Theta,\gamma,\lambda}&\bigg\{\phi_\cT(\gamma)+4{\mn^2[\delta]}\Tr(\Theta):
\Theta\succeq0,\gamma\geq0,\lambda\geq0,\\
&\left[\begin{array}{c|c}A^T\Theta A+\sum_k\gamma_kT_k-U&0_{n\times p}\cr\hline
 0_{p\times n}&0_{p\times p}\cr\end{array}\right]+\sum_\imath \lambda_\imath S_{\imath\ell}\succeq0\Bigg\}\nonumber
 \end{align}
 (cf. (\ref{U})),
 so that $\Opt_{\ell}(U,\delta)$ are efficiently computable convex real-valued functions of $U\in\bS^n$.
\item
Given $\theta\in[1,2]$, let $\vartheta={\theta\over 2-\theta}$, cf. Lemma \ref{lem-obs}. Consider  the ``master'' optimization problems
\begin{align}
\label{eqtolya}
\mR_{\ell}[\delta]=\inf_{G,U,\varsigma}
&\bigg\{2\left\|\rho^{(\ell)}[C,\delta]\right\|_{s,1}\rho_\ell[G,\delta]+\Opt_{\ell}(U,\delta):\\
&\quad\left.\begin{array}{l}
G\in\bR^{p\times n},\,U\in\bS^n,\,\varsigma\in\bR^n,\,{G^TC}\in\bS^n,\\
\left[\begin{array}{c|c}{G^TC}+U&B^T\cr\hline B&\Diag\{\varsigma\}\cr\end{array}\right]\succeq0,\,\|\varsigma\|_\vartheta\leq 1\end{array}\right\}.\nonumber
\end{align}
cf. (\ref{LMI}).
Let $(\ov G,\,\ov U,\,\ov\varsigma)$ be an optimal solution to the convex problem \rf{eqtolya}.
\begin{itemize}
\item[(a)] Let $f_{\ell j}^\pm$ be optimal solutions to problems  (\ref{newpsi}) specifying $\rho_{\ell j}^\pm[C,\delta]]$
and $h_{\ell j}^\pm$ -- the $\pi_\delta$-unit vectors such that $f_{\ell j}^\pm=\pi_\delta(f_{\ell j}^\pm)h_{\ell j}^\pm$; we denote   $H^{(a)}_\ell[\delta]$  the $m\times 2p$
matrix with columns $h_{\ell j}^\pm$, $j\leq p$.
\item[(b)] Let $f_{\ell j}^\pm$ be optimal solutions to problems  (\ref{newpsi}) specifying $\rho_{\ell j}^\pm[\ov G,\delta]$
and $h_{\ell j}^\pm$ be $\pi_\delta$-unit vectors such that $f_{\ell j}^\pm=\pi_\delta(f_{\ell j}^\pm)h_{\ell j}^\pm$; we denote   $H^{(b)}_\ell[\delta]$ the $m\times 2p$
matrix with columns $h_{\ell j}^\pm$, $j\leq p$.
\item[(c)] Let $\ov\Theta=\sum_i\alpha_if_ jf_i^T$ be the eigenvalue decomposition of $\ov U$, with orthonormal $[f_1,...,f_{m}]$, and let $H^{(c)}_\ell[\delta]$ be the matrix with the columns $h_i=f_i/{\mn[\delta]}$, $i\leq m$, so that $\pi_\delta(h_i)=1$, $i\leq m$
\end{itemize}
\end{enumerate}
\begin{proposition}\label{summary}
Let $H$ be a $(1-\epsilon)$-admissible contrast matrix containing, for some $\ell$ and $\delta$, the matrix
$[H^{(a)}_\ell[\delta],H^{(b)}_\ell[\delta],H^{(c)}_\ell [\delta]]$ as a submatrix.
Suppose that $x\in\cX^s$ and $\xi_x\in\Xi[x,H]$ are such that one of the cases $(\ell,\pm)$ occurs.
Then the recovery $\widehat{B}_H=B\wh x_H$ of $Bx$ is well defined and satisfies
\[
\left\|\widehat{B}_H-Bx\right\|_\theta^2\leq \mR_{\ell}[\delta].
\]
\end{proposition}
{\bf Proof.} As $H$ is $(1-\epsilon)$-admissible and $\xi_x\in\Xi[x,H]$, $\zeta=\wh x_H(Ax+\xi_x)-x$  is well-defined. By
 Lemma \ref{lem-obs} we have
 $$
\|B\zeta\|^2_{\theta}\leq \zeta^T\ov G^TC\zeta+\zeta^T\ov U\zeta.
$$
Now, by Observation \ref{verynew},
$$
\abs[\ov G\zeta]\leq\rho^{(\ell)}[\ov G,\delta], \quad 
\abs[C\zeta]\leq\rho^{(\ell)}[C,\delta],
$$
and, besides this, $\|C\zeta\|_1\leq 2\|C\zeta\|_{s,1}$. Hence,
\be
\zeta^T\ov G^TC\zeta\leq \|\ov G\zeta\|_\infty\|C\zeta\|_1\leq2\rho_\ell[\ov  G,\delta]\left\|\rho^{(\ell)}[C,\delta]\right\|_{s,1}.
\ee{eq23+}
Furthermore, by Proposition \ref{propnew} we have
$$
\zeta^T\ov U\zeta\leq {\Opt_\ell(\ov U,\delta)},
$$
implying that
$$
\|B\zeta\|_\theta^2\leq\mR_{\ell}[\delta],
$$
as claimed. \qed
\par
 By Proposition \ref{summary}, setting $M=p(4p+n)$, $\delta=\epsilon/M$ and specifying $m\times M$ matrix $H$ as the concatenation of matrices
 $[H^{(a)}_\ell[\delta],H^{(b)}_\ell[\delta],H^{(c)}_\ell [\delta]]$ for  $\ell\leq p$, we obtain a $(1-\epsilon)$-admissible contrast matrix $H$ such that $\|\cdot|_\theta$-
risk of recovering $Bx$ by $\widehat{B}_H=B\wh x_H$ satisfies
$$
\Risk^2_{\epsilon,\|\cdot\|_\theta}[\wh B_H|\cX^s]\leq \mR[\epsilon]:=\max_\ell\mR_{\ell}[\epsilon/M].
$$
When $x\in\cX^s$. $\xi_x\in\Xi[x,H]$ and one of the events $(\ell,\pm)$ takes place, the recovery error satisfies the bound
$$
\|\wh B_H-Bx\|_\theta^2\leq\mR_{\ell}[\epsilon/M].
$$
{
\subsubsection{Bounding the estimation error for a fixed $H$}
Contrast design procedure based on the signal decomposition described in Lemma \ref{lem-obs} can be modified to bound the $(s,\epsilon,\|\cdot\|_\theta)$-risk of the polyhedral estimate with a given contrast matrix.

Specifically, given $H\in \bR^{m\times M}$ we act as follows.
\begin{enumerate}
\item
For $G=[g_1,...,g_p]^T\in \bR^{p\times n}$ and $\ell,j=1,...,p$, let\footnote{This is nothing but (\ref{mrofGH}); we recall the notation here for reader's convenience.}
\begin{align*}
\mr_{\ell j}^{\chi}[G,H]&=\max_{z}\left\{[Gz]_j:\,z\in \cZ_\ell^\chi[H]\right\}
\quad\left[\ell\leq p,\,j\leq J,\,\chi\in\{-,+\}\right],\\
\mr_{\ell j}[G,H]&=\max\left[\mr_{\ell j}^-[G,H],\,\mr_{\ell j}^+[G,H]\right],\\
\mr^{(\ell)}[G,H]&=\left[\mr_{\ell 1}[G,H];...;\mr_{\ell J}[G,H]\right].
\end{align*}
Observe that $\mr^\pm_{\ell j}[G,\delta]$ and $\mr_{\ell j}[G,\delta]$ are efficiently computable convex in $G$ real-valued functions.
\item Next, for $\ell\leq p$ consider convex optimization problems
 \begin{align}\label{maineqU}
\ov\Opt_\ell[U,H]&=\min_{\gamma,\upsilon,\lambda}\Big\{\phi_{\cT}(\gamma)+4\sum_i\upsilon_i:\;\upsilon\geq 0,\,\gamma\geq 0,\,\lambda\geq 0,\\
&\qquad  \left.\left[\begin{array}{c|c}{A}^TH\Diag\{\upsilon\}H^T{A}+\sum_k\gamma_kT_k-U&0_{n\times p}\cr\hline
 0_{p\times n}&0_{p\times p}\cr\end{array}\right]+\sum_\imath \lambda_\imath S_{\imath\ell}\succeq0\right\}.\nonumber
\end{align}
Notice that $\ov\Opt_{\ell}[U,H]$ are efficiently computable convex real-valued functions of $U\in\bS^n$.
\item
Given $\theta\in[1,2]$ and $\vartheta={\theta\over 2-\theta}$, consider  optimization problems (cf. (\ref{LMI}))
\begin{align}
\label{eqtolyaU}
\ov\mR_{\ell}^H=\inf_{G,U,\varsigma}
&\bigg\{2\left\|\mr^{(\ell)}[C,H]\right\|_{s,1}\left\|\mr^{(\ell)}[G,H]\right\|_{\infty}+\ov\Opt_{\ell}(U,H):\\
&\quad\left.\begin{array}{l}
G\in\bR^{p\times n},\,U\in\bS^n,\,\varsigma\in\bR^n,\,{G^TC}\in\bS^n,\\
\left[\begin{array}{c|c}{G^TC}+U&B^T\cr\hline B&\Diag\{\varsigma\}\cr\end{array}\right]\succeq0,\,\|\varsigma\|_\vartheta\leq 1\end{array}\right\}.\nonumber
\end{align}

\end{enumerate}
\begin{proposition}\label{summaryU}
Let $H$ be a $(1-\epsilon)$-admissible contrast matrix.
Suppose that $x\in\cX^s$ and $\xi_x\in\Xi[x,H]$ are such that one of the cases $(\ell,\pm)$ occurs.
Then the recovery $\widehat{B}_H=B\wh x_H(Ax+\xi_x)$ of $Bx$ is well defined and satisfies
\be
\left\|\widehat{B}_H-Bx\right\|_\theta^2\leq \ov\mR_{\ell}^H.
\ee{Bbound00}
As a result, \[
\Risk^2_{\epsilon,\|\cdot\|_\theta}[\wh B_H|\cX^s]\leq \max_\ell\ov\mR_{\ell}^H.
\]
\end{proposition}
{\bf Proof.} Let $\ell\leq p$ be fixed, and $(\ov G,\,\ov U,\ov\varsigma)$ be a feasible solution to \rf{eqtolyaU}. Let also $\zeta=\wh x_H(Ax+\xi_x)-x$. By
 Lemma \ref{lem-obs} we have
\be
\|B\zeta\|^2_{\theta}\leq \zeta^T\ov G^TC\zeta+\zeta^T\ov U\zeta.
\ee{sdec000}
Applying \rf{forlvec} with $G=\ov G$ and $G=C$ we get
$$
\abs[\ov G\zeta]\leq\mr^{(\ell)}[\ov G,H], \quad 
\abs[C\zeta]\leq\mr^{(\ell)}[C,H],
$$
so that, cf. \rf{eq23+},
\[
\zeta^T\ov G^TC\zeta\leq \|\ov G\zeta\|_\infty\|C\zeta\|_1\leq 2\|\ov G\zeta\|_\infty\|C\zeta\|_{s,1}\leq2\left\|\mr^{(\ell)}[\ov  G,\delta]\right\|_\infty\left\|\mr^{(\ell)}[C,\delta]\right\|_{s,1}.
\]
On the other hand, when $(\ov\upsilon,\ov\gamma,\ov\lambda)$ is an optimal solution to \rf{maineqU} corresponding to $U=\ov U$, for $[\zeta;\eta]\in\mY_\ell[H]{\subset \cK_\ell}$ we have, by the semidefinite constraint of \rf{maineqU} (recall that under the premise of the proposition, $\|H^TA\zeta\|_\infty\leq 2$),
\begin{align*}
\zeta^T\ov U\zeta&\leq \sum_j\zeta^TAH\Diag\{\ov\upsilon_k\}H^TA\zeta +\sum_k\ov\gamma_k\zeta^TT_k\zeta+\sum_\imath\ov\lambda_\imath[\zeta;\eta]^TS_{\imath\ell}[\zeta;\eta]\\
&\leq4\sum_{j=1}^M\ov\upsilon_j+\phi_{\cT}(\ov\gamma)=\Opt_\ell[\ov U,H]
\end{align*}
(we use here the same argument as when verifying (\ref{chain}) in  the proof of Lemma \ref{prop:2rbound}).
When substituting the above bounds into \rf{sdec000}, we arrive at \rf{Bbound00}. The final risk bound is an immediate consequence of \rf{Bbound00} due to (!) and (!!).\qed
}
\section{Testing sparse hypotheses}
\label{sectsitgo}
We now move from recovering linear images of sparse signals to hypothesis testing. Our problem of interest now reads:
\begin{quotation}
\noindent Given a sensing matrix $A\in\bR^{m\times n}$ and observation
$
\omega=Ax+\xi_x$ (cf. \rf{obs})
of unknown signal $x$ we want to decide upon two hypotheses about $x$:
\begin{itemize}
\item $\mX^s$: $x$ belongs to a given convex and compact set $\cX\subset \bR^n$ and $C_\cX x$ is $s_x$-sparse,
\item $\mY^s$: $x$ belongs to a given convex and compact set $\cY\subset \bR^n$ and $\C_\cY y$ is $s_y$-sparse.
\end{itemize}
where $C_\cX\in\bR^{p_\cX\times n}$, $C_\cY\in\bR^{p_\cY\times n}$  are given matrices.
We assume that both $\cX$ and $\cY$ are nonempty, while $\mX^s\cap \mY^s=\emptyset$.
Our assumptions about the observation noise are the same as in Section \ref{sec:obss}.
\end{quotation}

\subsection{Building bricks}
Assume that we are given two nonempty convex compact sets $\cA$ and $\cB$, and given observation \rf{obs} we want to decide on the hypotheses $\mA=\{x\in \cA\}$ versus $\mB=\{x\in \cB\}$.
The building block of what follows is a simple test deciding on these hypotheses (a deterministic function of the observation taking values $\cA$ or
$\cB$, meaning that the respective hypothesis is accepted) which is as follows.
\paragraph{Test $\cT^\delta_{\cA, \cB}$.} Given design parameter $\delta\in(0,1),$ observation $\omega$, ``detector" $h\in\bR^m$, and threshold $\alpha\in\bR$, the test accepts hypothesis $\mA$ (and rejects hypothesis $\mB$) when $h^T\omega\geq\alpha$ and accepts $\mB$ (and rejects $\mA$) otherwise.
 We  specify $h=h_{\cA,\cB}^\delta$ as an optimal solution to the convex optimization problem
\[
\Opt=\Opt_{\cA,\cB}^\delta:=\max\left\{\min\limits_{x\in \cA} h^TAx-\max\limits_{y\in \cB}h^TAy: \pi_{\delta}(h)\leq1\right\}.
\]
When solving the problem, along with the optimal solution $h=h^\delta_{\cA,\cB}$, we get at our disposal the reals
$$
\overline{\Opt}=\min_{x\in \cA} [h_{\cA,\cB}^\delta]^TAx,\quad\underline{\Opt}=\max_{y\in \cB} [h_{\cA,\cB}^\delta]^TAy
$$
such that $\Opt=\ov\Opt-\ul\Opt$. We set
$$
\alpha[=\alpha_{\cA,\cB}^\delta] =\half (\overline{\Opt}+\underline{\Opt}).
$$
We define the risk of the test $\cT$ as the maximal over $x\in \cA\cup\cB$ $P_x$-probability of wrong decision:
\[
\Risk[\cT|(\cA,\cB)]=\max\left[\sup_{x\in \cA}\Prob_{\xi_x\sim P_x}\{\cT=\mB\},\,\sup_{y\in \cB}\Prob_{\xi_y\sim P_y}\{\cT=\mA\}\right].
\]
Let us make the following immediate
\begin{observation} \label{obsnewnew}Assume that
\begin{equation}\label{gap}
\Opt_{\cA,\cB}^\delta > 2.
\end{equation}
Then the test $\cT_{\cA,\cB}^\delta$ decides on the hypotheses $\mA$ and $\mB$ with risk $\leq\delta$.
\end{observation}
Indeed, let the signal $x$ underlying observation \rf{obs} obey hypothesis $\mA$, and let realization $\xi_x$ of the noisy component of $\omega$ satisfy $|[h_{\cA,\cB}^\delta]^T\xi_x|\leq1$, which happens with $P_x$-probability at least $1-\delta$ due to $\pi_{\delta}(h_{\cA,\cB}^\delta)\leq1$. In this case we have
\[[h_{\cA,\cB}^\delta]^T \omega=[h_{\cA,\cB}^\delta]^T Ax+[h_{\cA,\cB}^\delta]^T \xi_x\geq [h_{\cA,\cB}^\delta]^T Ax-1\geq\overline{\Opt}-1>\alpha
\]
where the concluding inequality is due to (\ref{gap}). Thus, in the situation in question (and for $x$ obeying $\mA$ it is of $P_x$-probability at least $1-\delta$), the test accepts hypothesis $\mA$. Similar reasoning demonstrates that when $x$ obeys $\mB$, the $P_x$-probability for $\mB$ to be accepted is $\geq 1-\delta$. \qed
\subsection{Testing hypotheses about sparse signals}
Now we are ready to
attack the problem posed at the beginning of Section \ref{sectsitgo}. Specifically, given problem's data (i.e., $A$, $\cX$, $\cY$,
{$C_\cX$, $C_{\cY}$,} $s_x$, and $s_y$)  and tolerance $\epsilon\in (0,1/2)$, we
\begin{itemize}
\item build $2p_\cX$ convex compact sets $\cX_{2i-1}=\cX^+_i$, $\cX_{2i}=\cX^-_i$, $i\leq p_\cX$. and $2p_\cY$ convex compact sets
$\cY_{2i-1}=\cY^+_i$, $\cY_{2i}=\cY^-_i$, $i\leq p_\cY$, where
\begin{align*}
\cX_i^\pm&=\left\{x\in \cX:\|C_\cX x\|_p\leq\pm s_{x}^{1/p} [C_\cX x]_i, 1\leq p\leq\infty\right\},\\
\cY_i^\pm&=\left\{y\in \cY:\|C_\cY y\|_p\leq\pm s_{y}^{1/p} [C_\cY y]_i\,1\leq p\leq\infty\right\},
\end{align*}
\item set $\delta=(2\max[p_\cX,p_\cY])^{-1}\epsilon$ and build the tests $\cT_{ij}=\cT^\delta_{\cX_i,\cY_j}$, $1\leq i\leq 2p_\cX,\;1
\leq
j\leq 2p_\cY$.
\end{itemize}
Our test $\ov\cT$ deciding on $\mX^s$ and $\mY^s$ is well defined in the {\sl ``good case"}, in which
$$
\Opt^\delta_{\cX_i,\cY_j}>2\;\forall (i\leq 2p_\cX,j\leq 2p_\cY)
$$
and is as follows: \begin{quotation}\noindent {\bf Test $\ov\cT$:} given observation $\omega$, the test
\begin{itemize}
\item accepts hypothesis $\mX^s$, if for some $i\leq 2p_\cX$ all tests $\cT_{ij}$, $j\leq 2p_\cY$,  as applied to $\omega$, accept hypothesis  $x\in \cX_i$;
\item accepts hypothesis $\mY^s$, if for some $j\leq 2p_\cY$  all tests $\cT_{ij}$, $i\leq 2p_\cX$,  as applied to $\omega$, accept hypothesis  $x\in \cY_j$;
\item fails to decide  when neither one of the  two situations above takes place.
\end{itemize}
\end{quotation}
Observe that in the ``good case" the test indeed is well defined -- the situation in which the above rules prescribe to accept both $\mX^s$ and $\mY^s$ is impossible. Indeed, in this situation there exist $\bar{i},\bar{j}$ such that the test $\cT^\delta_{\bar{i}\bar{j}}$ as applied to the observation at hand accepts the hypotheses that $x\in \cX_{\bar{i}}$ and $x\in \cY_{\bar{j}}$ simultaneously, which is impossible -- each test $\cT_{ij}$ accepts exactly one of the hypotheses it is applied to.
\par
The risk of the test $\ov\cT$ we have just designed is described by the following
\begin{proposition}\label{proprisk}  In the situation of this section, assume that ``good case" occurs. Then the risk of the test $\ov\cT$ is at most $\epsilon$.
\end{proposition}{\bf Proof.} Indeed, let the signal $x$ underlying the observation  obey the hypothesis $\mX^s$ and let $\iota$ be the index of the largest in magnitude entry in $x$. As $x$ is $s_x$-sparse and belongs to $\cX$, it either belongs to $\cX^+_\iota$, or to $\cX^-_\iota$, and thus $x\in \cX_{\ov i}$ for some $\ov i=\ov i(x)$. As we are in the ``good case," when $\cA$ is $\cX_{\ov i}$ and $\cB$ is $\cY_j$ for $j\leq 2p_\cY$, the premise of Observation \ref{obsnewnew} holds, so that the $P_x$-probability for test $\cT^\delta_{\bar{i}j}$ to accept the hypothesis $x\in\cX_{\bar{i}}$ is at least $1-\delta$.  When the latter event takes place for all $j\leq 2p_\cY$, the test $\ov\cT$ accepts hypothesis $\mX^s$;  as a result, the $P_x$-probability for $\ov\cT$ to accept $\mX^s$ is at least $1-2p_\cY \delta\geq1-\epsilon$. Similar reasoning demonstrates  that in the ``good case," when $x$ obeys $\mY^s$, the latter hypothesis will be accepted by $\ov\cT$ with $P_x$-probability at least $1-\epsilon$.
\qed
\paragraph{Remark.} Proposition \ref{proprisk} does not apply beyond the ``good case." Note that if $A\cX_i$ and $A\cY_j$ share a point for some $i, j$, we enter a ``bad case" where $\Opt^\delta_{\cX_i,\cY_j} = 0$, rendering a ``good case" impossible whatever is number $\cK$ of repeated observations available. On the other hand, suppose that we are not in the ``bad case," and that $\cK$-repeated observations are available, that is, our observation $\omega$ is given by (\ref{aggr}) with $\cK$ under our control. Assume also that the observation scheme is either sub-Gaussian, or Gaussian, or Discrete, or  Poisson. As it is immediately seen from the results on the repeated observations presented in  Section \ref{sec:obss}, passing from single-observation to $\cK$-repeated observation reduces norms $\pi_{\delta}$ by factor
$\sqrt{\cK}$, that is, increases $\Opt^\delta_{\cX_i,\cY_j}$ by factor at least $\sqrt{\cK}$. Because we are not in the ``bad case," single-observation quantities $\Opt^\delta_{\cX_i,\cY_j}$, $\delta=\epsilon/(2n)$, are positive for all $1\leq i,j\leq 2n$. Therefore, we can efficiently find the minimal $K$ which ensures that the ``good case" holds for a $K$-repeated observation, allowing us to decide between hypotheses $\mX^s$ and $\mY^s$ with risk at most $\epsilon$.

\section{Numerical illustrations}
We present here results of several ``proof of concept" simulation experiments illustrating numerical performance of the recovery and testing procedures described in Sections \ref{Sec1}--\ref{sectsitgo}.
\subsection{Estimating linear functionals}\label{sec:numeric1}
In this section we present results of three numerical experiments. The setup of the first experiment is as follows.
We put $\cX=\{x\in \bR^n:\,\|x\|_\infty\leq 10\}$, $C=I_n$,  and generate a random $A$: we draw an $m\times n$ matrix from Gaussian ensemble and scale it to have unit column norms.
Using the technique described in \cite{juditsky2011verifiable} we evaluate the ``goodness" of $A$ by computing certified lower bound $\underline s(A)$ and upper bound $\ov s(A)$ on the maximal ${s=s_*}$ such that $A$ is $s$-good (i.e., such that $\ell_1$-recovery of any $s$-sparse $x$ from the noiseless observation $\omega=Ax$  is exact). We assume that our observation is $\omega=Ax+\xi$ with $x\in \cX^s$ and Gaussian $\xi\sim \cN(0,\sigma^2)$ for $\sigma=0.01$. For different values of $s\leq \underline s(A)$ we compute four upper bounds  for the risks of the polyhedral recovery \rf{polyest} of the entries $x_j=e_j^Tx$, $j=1,...,n$, the risk reliability parameter being set to $\epsilon=0.05$:
\begin{enumerate}
\item the bound by Proposition \ref{obslin} for the ``vanilla" Dantzig Selector (DS) (estimates $e_j^T\widehat{x}_H(\cdot)$, see \rf{polyest}, with contrast matrix $H=\varkappa^{-1} A$, $\varkappa=\sigma\chi_{\epsilon/n}$);
\item we use the obtained bounds $\mr_\infty$ and $\mr_1$ for the $\ell_\infty$ and $\ell_1$ risks of Dantzig Selector to build a reduced complexity estimate of $x_j$ from Section \ref{sec:rce} when utilizing the set
    \[\ov\cZ=\{z\in \ov\cX:\,\|z\|_\infty\leq \mr_\infty,\,\|z\|_1\leq \mr_1\}\]
    as the localizer for the error of DS estimate and compute the corresponding risk bounds (estimate ``Reduced DS");
\item We build the reduced complexity estimates of $x_j$ (with combined contrasts $\ov H=[H^{(1)},h_j]\in\bR^{m\times n+1}$) following the ``complete recipe" in Section \ref{sec:rce} and compute the corresponding risk bounds (referred to as ``Simple Polyhedral");
\item Finally, we compute the ``full'' contrasts $ \ov H^{(j)}\in\bR^{m\times 2n}$ and the bounds $\rho^{(j)}$ for recovery of the entries $x_j$ of $x$  utilizing the construction described in Section \ref{sec:designlin1} (denoted ``Polyhedral").
\end{enumerate}
We present here the results for a realization of a $48\times 64$ matrix $A$ for which we certify $\underline s(A)=4$ and $\ov s(A)=6$ (for this matrix, the bound by Mutual Incoherence does not allow one to certify $s$-goodness even for $s=1$).
Results for $s=4$ are presented in Figure \ref{fig:lfe2}. In this case, computed upper bounds on the $\|\cdot\|_2$-risk are $54.9092$ for Dantzig Selector, $9.2918$ for the reduced complexity polyhedral estimate, and $7.6836$ for the ``full" polyhedral estimate.\footnote{Notice that a straightforward bound for the $\|\cdot\|_2$-norm of the estimation error in this case is $20\sqrt{8}=56.5685$.}
Plots in Figure \ref{fig:lfe1} illustrate the results of computation for $s=3$ (left plot) and $s=2$ (right plot). The induced upper bounds on the $\|\cdot\|_2$-risk are $6.6966$ ($0.7727$) for Dantzig Selector, $0.7581$ ($0.2514$) for the reduced complexity polyhedral estimate, and $ 0.7458$ ($0.2500$) for the ``full" polyhedral estimate for $s=3$ (respectively, $s=2$).
\begin{figure}[h]
\begin{tabular}{cc}
\hspace{-0.5cm}\includegraphics[width=0.5\textwidth]{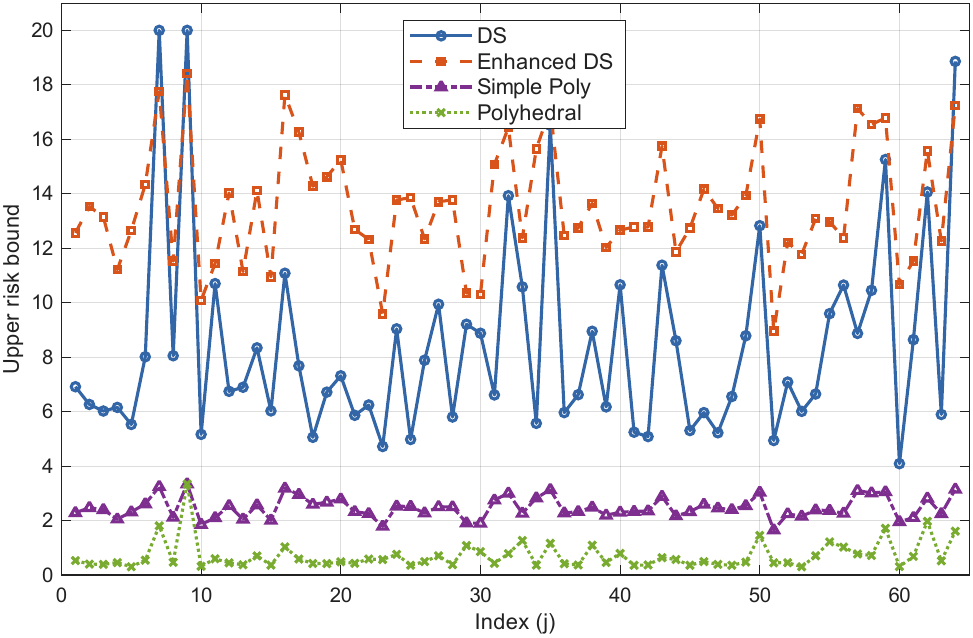}&
\includegraphics[width=0.5\textwidth]{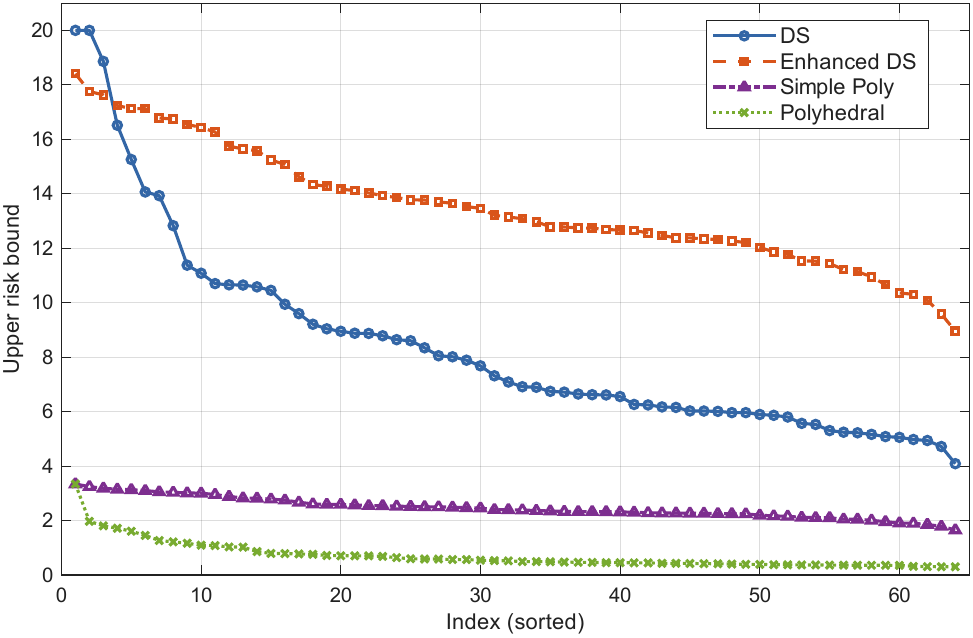}
\end{tabular}
\caption{\label{fig:lfe2} Bounds for risks of recovery of the entries $x_j$, $j=1,...,64$, sparsity parameter $s=4$.  }
\end{figure}

\begin{figure}[h]
\begin{tabular}{cc}
\hspace{-0.5cm}\includegraphics[width=0.5\textwidth]{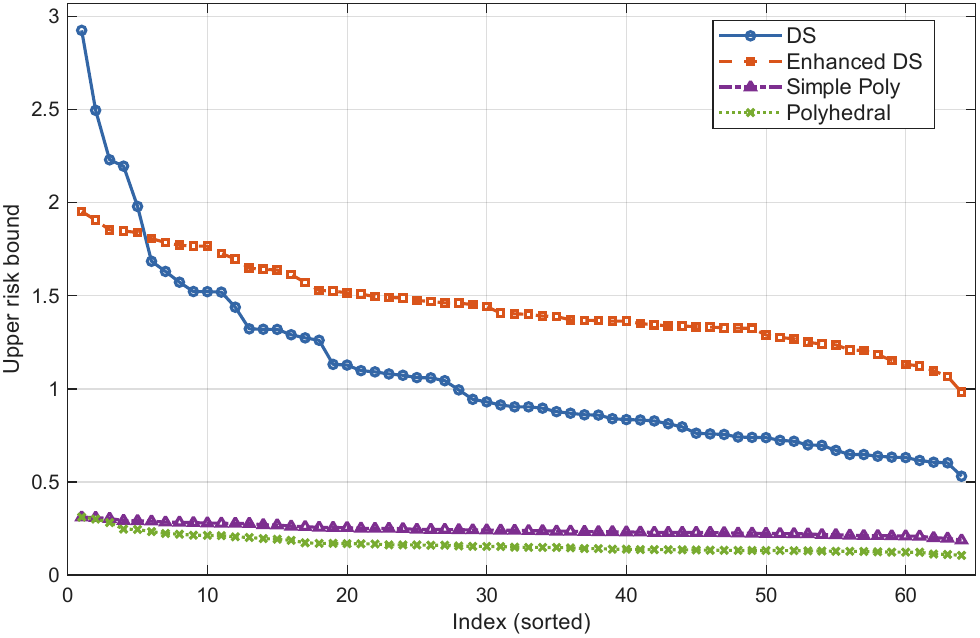}&
\includegraphics[width=0.5\textwidth]{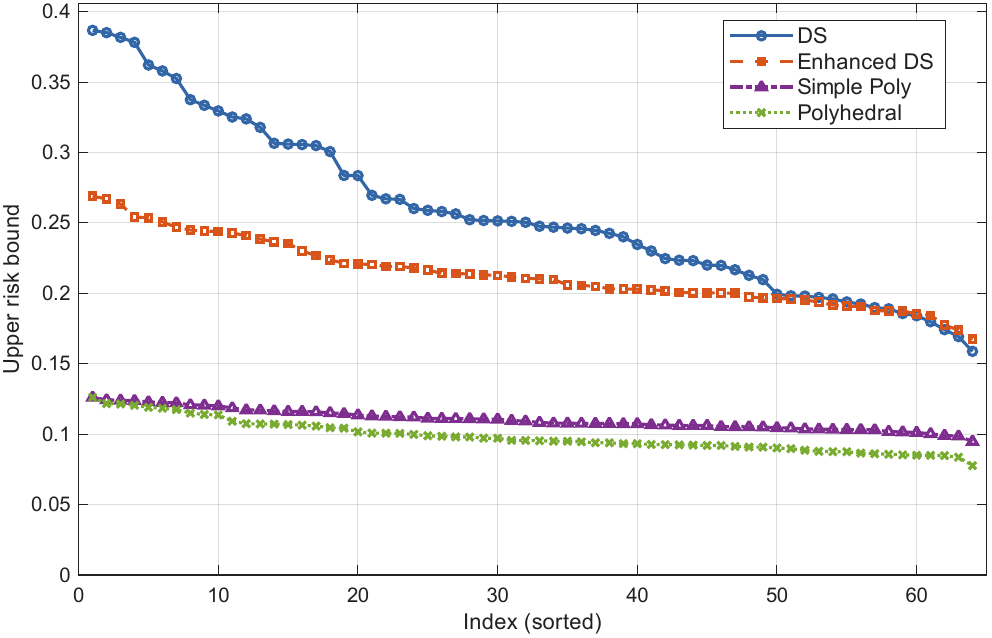}
\end{tabular}
\caption{\label{fig:lfe1} Bounds for risks of recovery of the entries $x_j$, $j=1,...,64$. Left plot: sparsity $s=3$; right plot:sparsity $s=2$.  }
\end{figure}
In our second experiment, in the same setup, we compare numerically the bounds in \rf{frho} and  \rf{opti3}. In Figure \ref{fig:55} we present the computed bounds for the polyhedral recoveries ("Enhanced DS" and "Simple Polyhedral") along with the corresponding bounds for debiased estimates ("Debiased DS" and "Debiased Polyhedral") for the values of the sparsity parameter $s=3$ and $s=2$.

\begin{figure}[h]
\begin{tabular}{cc}
\hspace{-0.5cm}\includegraphics[width=0.5\textwidth]{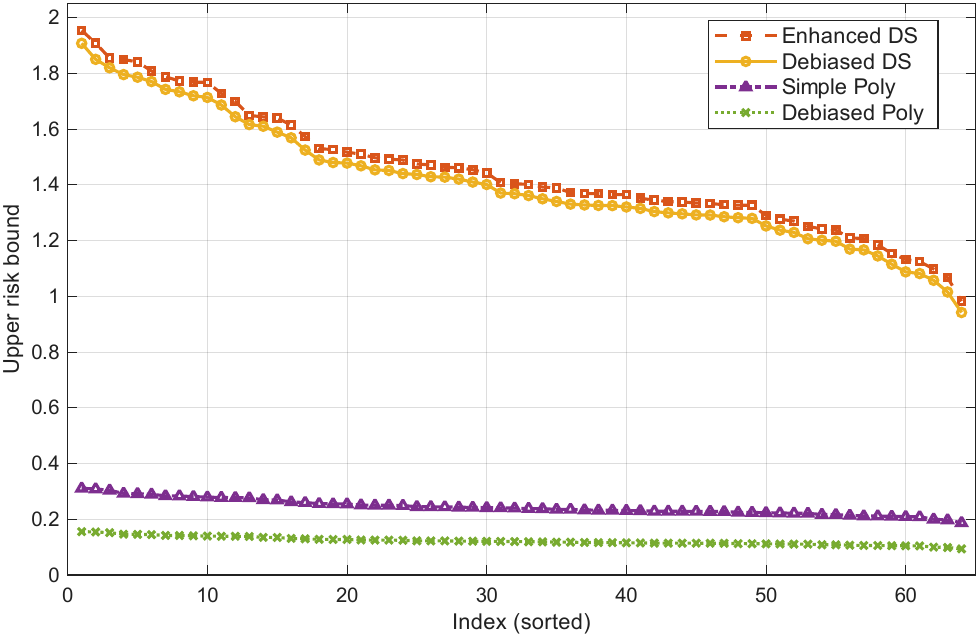}&
\includegraphics[width=0.5\textwidth]{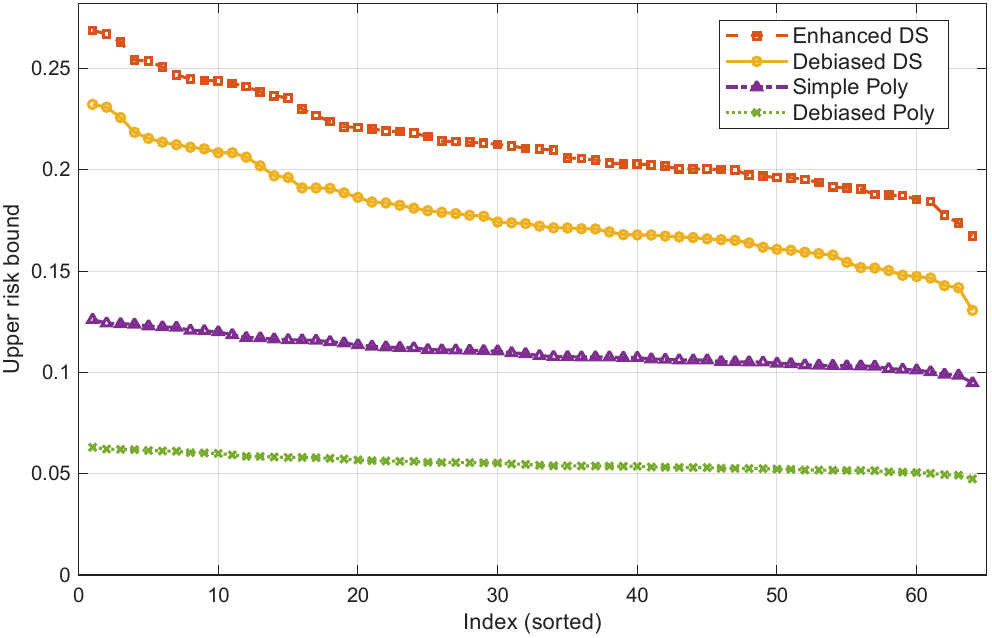}
\end{tabular}
\caption{\label{fig:55} Upper bounds for risks of recovery of the entries $x_j$, $j=1,...,64$, of the sparse signal in the setup of the first experiment in Section \ref{sec:numeric1}. Left plot: sparsity parameter $s=3$; right plot:  sparsity parameter $s=2$.}
\end{figure}
The setup of the last experiment of this section is similar to that of the first one. We select the same signal set $\cX=\{x\in \bR^n:\,\|x\|_\infty\leq 10\}$ and $C=I_n$, and generate $K=64$ sensing matrices $A[k]$ from the Gaussian ensemble in $\bR^{m\times n}$ which are then scaled to have unit column norms; we also draw $K$ random vectors $g[k]$ with unit Euclidean norm. For each $A[k]$ we compute the bound $\underline s(A[k])$ for the $s$-goodness parameter $s_*$ and compute upper bounds \rf{risklinest} for the risks of the polyhedral estimates of the values $g[k]^Tx$ of unknown $x\in \cX^s$ from observation $\omega=Ax+\xi$. Here $\xi\sim \cN(0,\sigma^2)$ with $\sigma=0.01$ and the risk reliability parameter is set to $\epsilon=0.05$. We compute bounds for the risk of recovery of $g[k]^Tx$ by $g[k]^T\wh x_{DS}$ where $\wh x_{DS}$ is the Dantzig Selector estimate of $x$, risk bound by Proposition \ref{prop:Gred} for the ``reduced complexity'' polyhedral estimate (Simple Polyhedral estimate), and the bound \aic{by Proposition \ref{aiprop1}}{\rf{claimthat}} for the polyhedral estimate in Section \ref{sec:designlin1} (Polyhedral estimate). Figures \ref{fig:lfe11} and \ref{fig:lfe12} show the results for $s = \underline{s}(A[k])$ and $s = \underline{s}(A[k]) - 1$ (here $\underline s(A[k])=3$ for 2 realizations of $A[k]$, and $\underline s(A[k])=4$ for the remaining ones).

\begin{figure}[h]
\begin{tabular}{cc}
\hspace{-0.5cm}\includegraphics[width=0.49\textwidth]{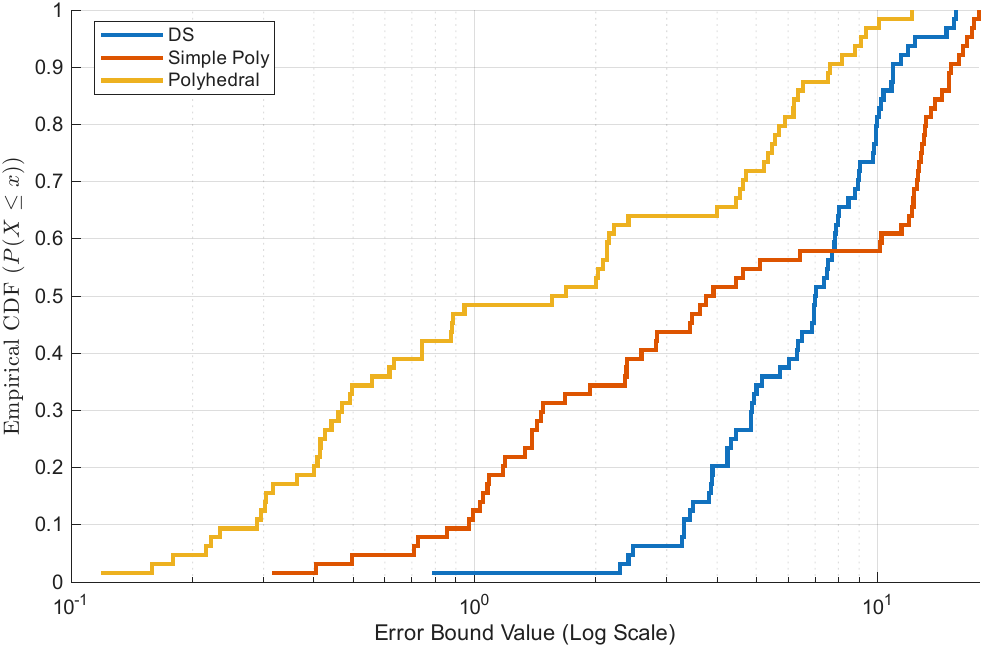}&
\includegraphics[width=0.51\textwidth]{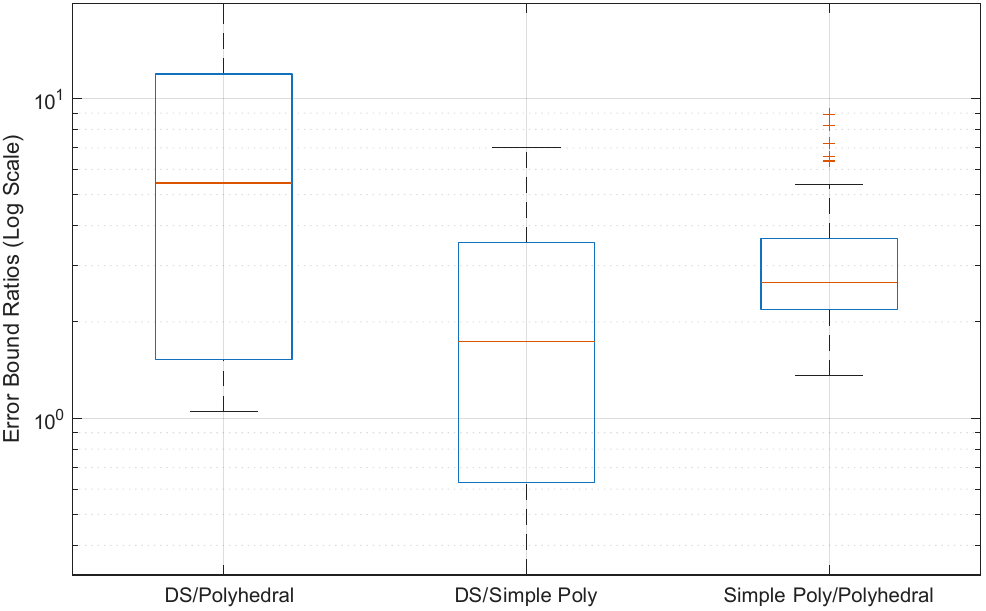}
\end{tabular}
\caption{\label{fig:lfe11} Upper bounds for risks of estimation of random linear functionals $g^Tx$, sparsity parameter $s=\underline s(A[k])$. Left plot: cumulative distributions of the risk bounds; right plot: distributions of the ratios of the risk bounds.}
\end{figure}

\begin{figure}[h]
\begin{tabular}{cc}
\hspace{-0.5cm}\includegraphics[width=0.49\textwidth]{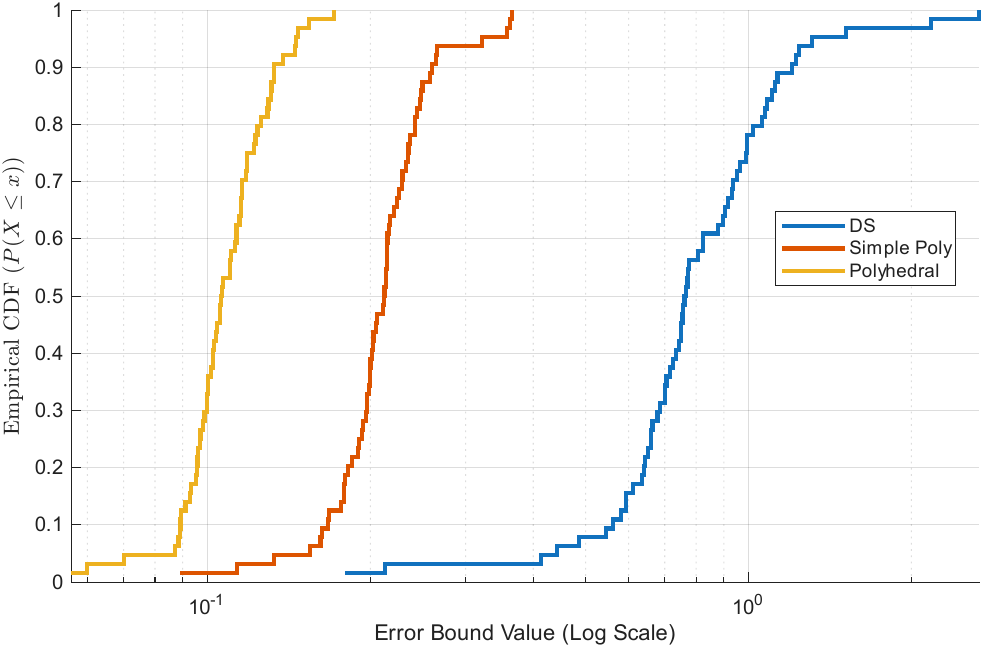}&
\includegraphics[width=0.51\textwidth]{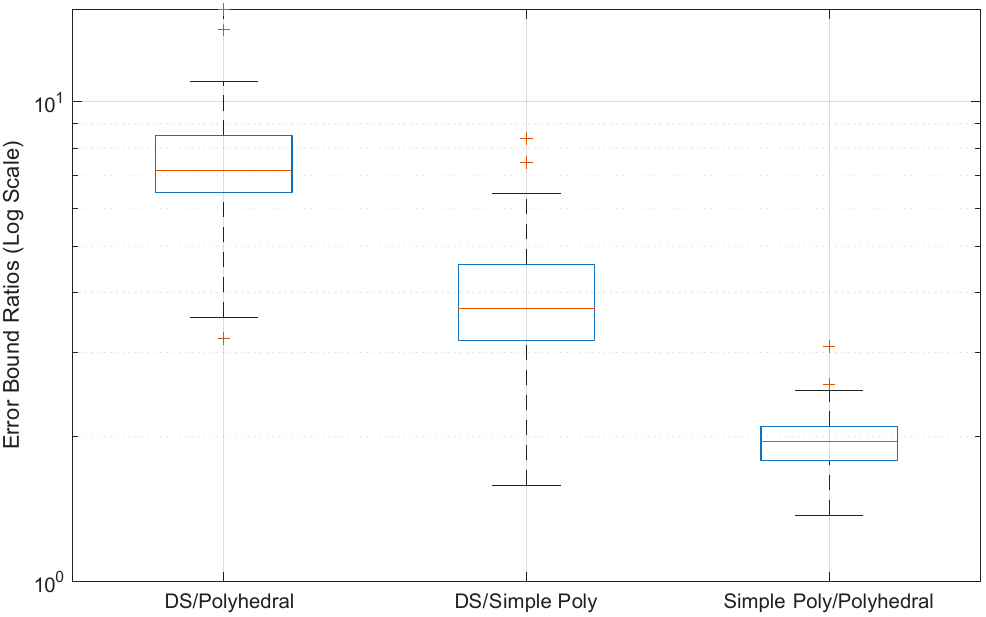}
\end{tabular}
\caption{\label{fig:lfe12} Upper bounds for risks of estimation of random linear functionals $g^Tx$, sparsity parameter $s=\underline s(A[k])-1$. Left plot: cumulative distributions of the risk bounds; right plot: distributions of the ratios of the risk bounds.}
\end{figure}

\subsection{Sparse signal recovery}
We now present results of ``proof of concept" simulation experiments for the recovery procedures proposed in
Sections \ref{sec:3} and \ref{Sec2b}. In these experiments
\begin{itemize}
\item $\cX=\{x\in\bR^{32}:\,\|x\|_\infty\leq10\}$;
\item a $28\times 32$ matrix $A$ is drawn at random from the Gaussian ensemble and then normalized to have columns of unit Euclidean length;
this matrix turns out to be provably $4$-good;
\item $C=B=I$, the unit matrix, the error of recovery of $Bx=x$ was measured in $\|\cdot\|_2$ in the first series of experiments and in $\|\cdot\|_1$ in the second series;
\item we have computed estimates described in Sections \ref{sec:3} and \ref{Sec2b} for the values $s=3$ and $s=5$ of the sparsity parameter. We set the reliability parameter $\epsilon=0.05$ assuming Gaussian observation scheme with $\sigma=0.01$.
\end{itemize}
The results of experiments are summarized in Figures \ref{fig21}, \ref{fig22} for $s=3$ and in Figures \ref{fig23}, \ref{fig24} for $s=5$. Each boxplot represents the actual recovery errors in the series of 1000 experiments with randomly generated signals of prescribed sparsity; horizontal red bars represent the corresponding risk bounds
\par
As a benchmark, we present performance data for the Dantzig Selector (DS)---that is, the polyhedral estimate with
$H$ obtained from $A$  by scaling the columns to unit $\pi_{\epsilon/n}$-norm---together with the upper risk bound $\max_\ell\ov\Opt_\ell[H]$ (see Lemma \ref{prop:2rbound}). For comparison: under the circumstances, the theoretical upper bound on the  $(\epsilon=0.05,\,\|\cdot\|_2)$-risk of recovering {\sl all}
signals from $\cX$ (including nonsparse ones) via the optimal polyhedral estimate is  91.709 (see \cite[Section 5.1]{PUP}).
\aic{
\begin{table}
{\scriptsize$$
\begin{array}{||c||c|c|c||c||c||c||c|c|c||c||}
\cline{1-5}
\cline{7-11}
\cline{1-5}
\cline{7-11}
\hbox{Estimate}&\hbox{median}&\hbox{mean}&\max&\hbox{Risk}&&\hbox{Estimate}&\hbox{median}&\hbox{mean}&\max&\hbox{Risk}\\
\cline{1-5}
\cline{7-11}

\hbox{DS}&0.064 & 0.072 &0.213 &  8.155&&\hbox{DS}&0.108 & 0.114& 0.314 &53.44\\
\cline{1-5}
\cline{7-11}
\hbox{Prop.\ }\ref{bxinf3}&
& & &  1.956&&\hbox{Prop.\ }\ref{bxinf3}& & & &62.16  \\
\cline{1-5}
\cline{7-11}
\hbox{Prop.\ } \ref{prop:reduce1}&
0.098 & 0.102& 0.155 &  1.520&&\hbox{Prop.\ } \ref{prop:reduce1}&
0.167&  0.168& 0.250 &63.25\\
\cline{1-5}
\cline{7-11}
\hbox{Prop.\ } \ref{prop:Gred}& 0.109 & 0.112 &0.170 &  1.080&&\hbox{Prop.\ } \ref{prop:Gred}&0.167 & 0.168& 0.250 &63.25\\
\cline{1-5}
\cline{7-11}
\hbox{Obs.\ }
 \ref{obsnew}& 0.070 & 0.072& 0.139  & 1.405&&\hbox{Obs.\ }
 \ref{obsnew}&0.122 & 0.129& 0.300 &50.40\\
\cline{1-5}
\cline{7-11}
\hbox{Prop.\ } \ref{summary}&0.072&  0.075& 0.132 &  0.979&&\hbox{Prop.\ } \ref{summary}&0.124 & 0.129 &0.268&9.312\\
\cline{1-5}
\cline{7-11}
\multicolumn{5}{c}{s=3}&\multicolumn{1}{c}{}&\multicolumn{5}{c}{s=5}\\
\end{array}
$$}
\caption{\label{tables3} Recovery errors for $\|\cdot\|_2$-recovery, data over 100 simulations, and theoretical upper bounds on $(0.05,s,\|\cdot\|_2)$-risk.}
\end{table}
 {\crd Arik, why no bound "by Proposition 4.3"? Propositions 4.5 and 4.6 both describe "reduced complexity estimate". What is the difference between the two?}
\begin{table}
{\scriptsize$$
\begin{array}{||c||c|c|c||c||c||c||c|c|c||c||}
\cline{1-5}
\cline{7-11}
\cline{1-5}
\cline{7-11}
\hbox{Estimate}&\hbox{median}&\hbox{mean}&\max&\hbox{Risk}&&\hbox{Estimate}&\hbox{median}&\hbox{mean}&\max&\hbox{Risk}\\

\cline{1-5}
\cline{7-11}
\hbox{DS}&  0.106& 0.129& 0.465 &32.79&&\hbox{DS}&       0.211 & 0.272 &0.960 &287.9\\
\cline{1-5}
\cline{7-11}

\hbox{Prop.\ } \ref{prop:reduce1}&
0.164 & 0.170& 0.266 &3.722&&\hbox{Prop.\ } \ref{prop:reduce1}&
0.202&  0.223& 0.539&200.0\\
\cline{1-5}
\cline{7-11}
\hbox{Prop.\ } \ref{prop:Gred}&0.184 &0.188 &0.289&3.722&&\hbox{Prop.\ } \ref{prop:Gred}&0.212&0.218& 0.333& 200.0\\
\cline{1-5}
\cline{7-11}
\hbox{Obs. }
 \ref{obsnew}& 0.133 & 0.140& 0.324 &6.666&&\hbox{Obs.\ }
 \ref{obsnew}& 0.096 & 0.100 &0.210 &267.4\\
\cline{1-5}
\cline{7-11}
\hbox{Prop.\ } \ref{summary}& 0.129 & 0.131& 0.257 &5.246&&\hbox{Prop.\ } \ref{summary}&0.233 & 0.248& 0.582  & 48.28\\
\cline{1-5}
\cline{7-11}
\multicolumn{5}{c}{s=3}&\multicolumn{1}{c}{}&\multicolumn{5}{c}{s=5}\\
\end{array}
$$}

\caption{\label{tables4} Recovery errors for $\|\cdot\|_1$-recovery, data over 100 simulations, and theoretical upper bounds on $(0.05,s,\|\cdot\|_1)$-risk.}
\end{table}
}
{
\begin{figure}[h]
\begin{center}
\includegraphics[width=0.8\textwidth]{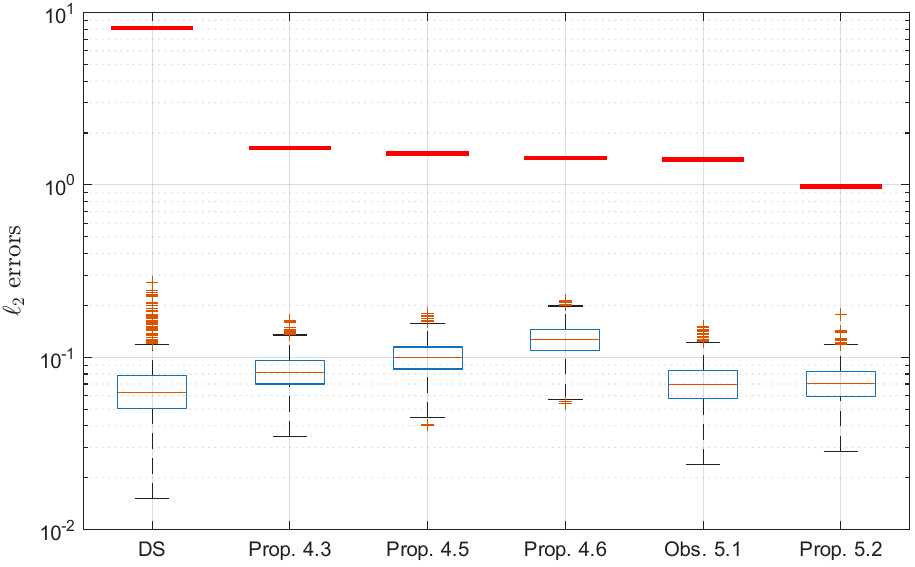}
\end{center}
\caption{\label{fig21} $\ell_2$-estimation errors and upper risk bounds for recovery of random $3$-sparse ($s=3$)signals using routines from Sections \ref{sec:3} and \ref{Sec2b}. Red bars represent the corresponding risk bounds.}
\end{figure}
\begin{figure}[h]
\begin{center}
\includegraphics[width=0.8\textwidth]{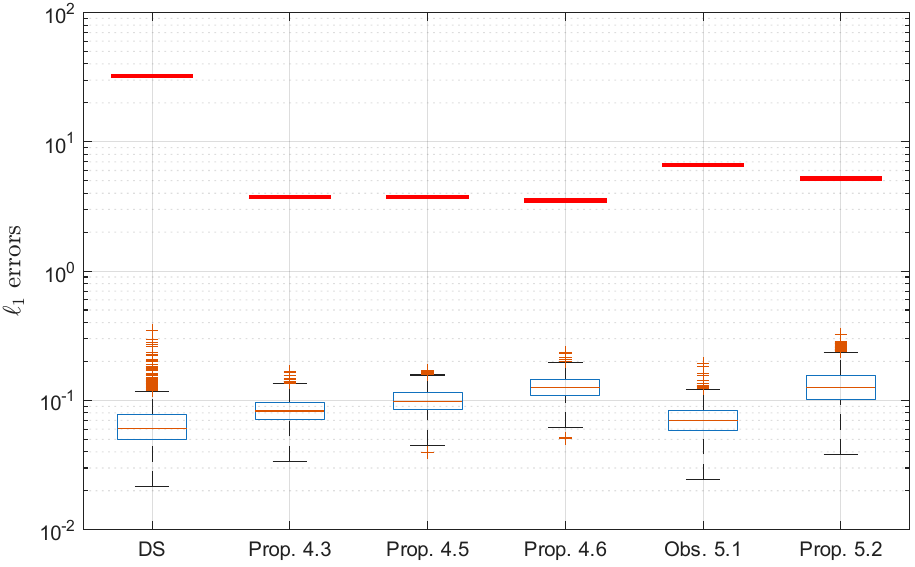}
\end{center}
\caption{\label{fig22} $\ell_1$-estimation errors and upper risk bounds for recovery of random $3$-sparse ($s=3$) signals using routines from Sections \ref{sec:3} and \ref{Sec2b}}
\end{figure}
\begin{figure}[h]
\begin{center}
\includegraphics[width=0.8\textwidth]{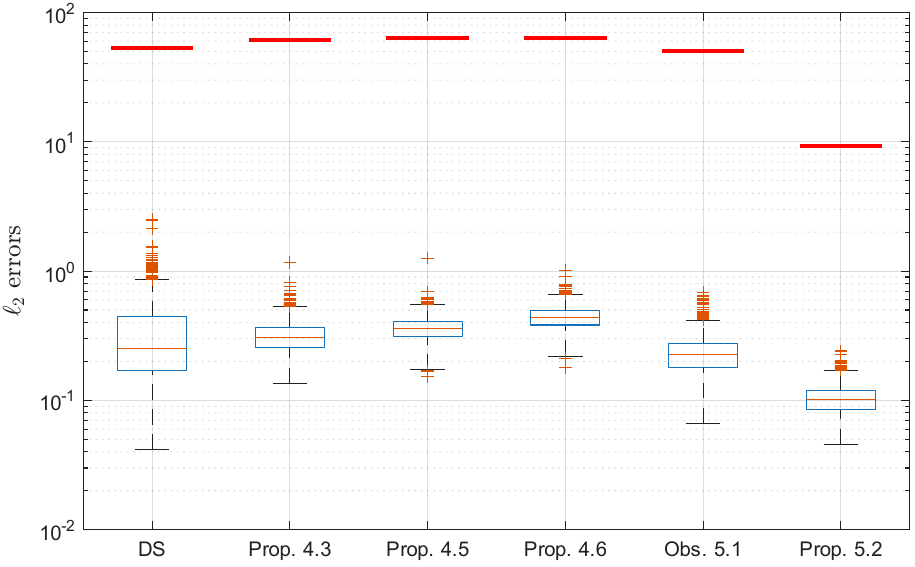}
\end{center}
\caption{\label{fig23} $\ell_2$-estimation errors and upper risk bounds for recovery of random $5$-sparse  ($s=5$) signals using routines from Sections \ref{sec:3} and \ref{Sec2b}.}
\end{figure}
\begin{figure}[h]
\begin{center}
\includegraphics[width=0.8\textwidth]{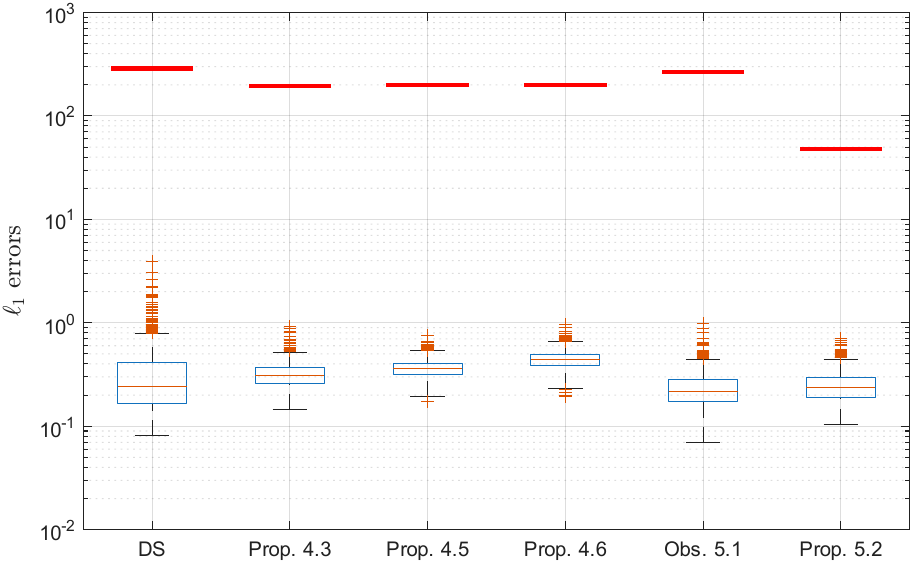}
\end{center}
\caption{\label{fig24} $\ell_1$-estimation errors and upper risk bounds for recovery of random $5$-sparse ($s=5$) signals using routines from Sections \ref{sec:3} and \ref{Sec2b}}
\end{figure}

}

\subsection{Testing sparse hypotheses}
We consider Gaussian observation scheme with observation noise $\xi\sim\cN(0,I_n)$ and
\[C=A=I_n,\;\cX=\{x\in\bR^n_+: \,n\leq \sum_ix_i\leq 2n\}, \;\cY=\{y\in\bR^n_+:\,y_i\leq 2,\,i\leq n\}.
\]
In the present situation, sets $\cX$ and $\cY$ have a massive intersection containing the box $\{x:1\leq x_i\leq 2\,\forall i\}$. Thus, there is no nontrivial test deciding upon the hypotheses $x\in\cX$ and $x\in\cY$ however large the number  $\cK$ of observations may be. The situation changes dramatically when passing to sparse signals: assuming $x$ to be $s$-sparse, one has \[
\cX_i\subset\left\{x\in\bR^n_+ :\,\sum_kx_k\geq n\right\}
\] and
\[\cY_j\subset\left\{y\in\bR^n_+:\,\sum_ky_k\leq s\|y\|_\infty\leq 2s\right\}.
\]
When $s<2n$, one has $\Opt^\delta_{\cX_i,\cY_j}>0$ for all $i,j$ already when $\cK=1$.
Immediate computation shows that when $\epsilon=0.01$, $n=100$, and $s=10$, only one ``individual observation" (\ref{ind}) is  sufficient to ensure that ``good case" holds, allowing  to decide on $\mX^s$ vs. $\mY^s$ with risk $\leq 0.01$.
When $s$ is increased to $40$, $\cK=17$ individual observations suffice to achieve the same risk of test.
When $n=10\,000$ and $s=4\,900$, $\cK=26$ individual observations are enough, and  a subsequent drop in $\epsilon$ from 1e-2 to 1e-4 increases $\cK$ to 35.
This is fully supported by numerical evidence: in a series of 1\,000 simulations with $\epsilon=1e-4$, $n=10\,000$, $s=4\,900$, and $\cK=35$, not a single wrong decision was observed.\footnote{The reader could ask how ``computation-friendly" could be a procedure which requires computing $(2n)^2=4{\rm e}8$ separators $h^\delta_{\cX_i,\cY_j}$ and then to evaluate the corresponding linear forms on the available observation. As it happens, our toy problem is permutationally symmetric in  signals obeying the hypotheses in question; as a result, it suffices to compute just one  separator (it takes seconds).
Then it takes  about 2--3 sec to apply the resulting test to the observation.}
\appendix
\section{Miscellaneous proofs}

We start with the following well-known statement.
\begin{lemma}\label{simple1} Let $y,y'\in \bR^n$ with $s$-sparse $y$ and $\|y'\|_1\leq \|y\|_1$. Then for $z=y'-y$ and $1\leq p\leq\infty$ it holds
\begin{equation}\label{simple}
\|z\|_p\leq 2^{1/p}\|z\|_{s,p}
\end{equation}
whence also
\begin{equation}\label{simplea}
\|z\|_p\leq (2s)^{1/p}\|z\|_\infty.
\end{equation}
\end{lemma}
{\bf Proof.} Indeed, let us first check that (\ref{simple}) holds true for $p=1$. Let $I$ be the set of indices of nonzero entries in $y$.
We have
\[
\sum_{i\not\in I}|z_i|=\sum_{i\not\in I}|y^\prime_i|
=\|y'\|_1-\sum_{i\in I}|y^\prime_i|\leq \|y\|_1-\sum_{i\in I}|y^\prime_i|=\sum_{i\in I}|y_i|-\sum_{i\in I}|y^\prime_i|\leq\sum_{i\in I}|z_i| .
\]
Hence,
\be
\|z\|_1=\sum_{i\not\in I}|z_i|+\sum_{i\in I}|z_i|\leq 2\sum_{i\in I}|z_i|\leq 2\|z\|_{s,1},
\ee{z1}
 as claimed. Now let $1<p<\infty$ (the case of $p=\infty$ is trivial). Assuming  w.l.o.g. that $|z_1|\leq|z_2|\leq...
\leq |z_n|$, we have for $i\leq n-s$
\[|z_i|\leq |z_{n-s+1}|\leq s^{-1}\|z\|_{s,1},
\] and also $\sum_{i=1}^{n-s}|z_i|\leq \|z\|_{s,1}$ by \rf{z1}. Thus,
$$
\sum_{i=1}^{n-s}|z_i|^p\leq \left[{\max}_{i\leq n-s}|z_i|\right]^{p-1}
{\sum}_{i=1}^{n-s}|z_i|\leq \|z\|_{s,1}^ps^{1-p}=\left[{\sum}_{i=n-s+1}^n|z_i|\right]^ps^{1-p}
\leq {\sum}_{i=n-s+1}^n|z_i|^p,
$$
implying that
\[\|z\|_p^p\leq 2{\sum}_{i=n-s+1}^n|z_i|^p=2\|z\|_{s,p}^p.\fqed\]

\begin{lemma}\label{lem:a1} Let $\xi_i\sim\SG(\mu_i,\rho^2_i I_m)$ (i.e., sub-Gaussian with parameters $\mu_i$ and $\rho^2_i I_m$) random vectors, $0\leq \rho_i\leq \sigma$,  and let $x$ be a probabilistic vector. Let also $\mu=\mu(x):=\sum_j\mu_jx_j$ and let $\varrho$ 
satisfy
\be
\forall(x\in\cX,f\in\bR^m):\quad \sum_ix_i\exp\left\{\tfrac{2}3[f^T(\mu(x)-\mu_i)]^2\right\}\leq \exp\left\{\half\varrho^2f^Tf\right\},
\ee{rho0}
e.g., $\varrho={2\over\sqrt{3}}\max_{i,j}\|\mu_i-\mu_j\|_2$.
Then the mixture  $\xi_x$ of random variables $\xi_i-\mu$ with coefficients $x_i$ is sub-Gaussian with parameters $0$ and $[\sigma^2+\varrho^2]I_m$.
\end{lemma}
{\bf Proof.} Indeed, let $f\in\bR^m$. Then
\begin{align*}
\bE\{\exp\{f^T\xi_x\}\}&=\sum_ix_i\bE\{\exp\{f^T[\xi_i-\mu]\}\}\\
&=\sum_ix_i\exp\{f^T[\mu_i-\mu]\}\bE\{\exp\{f^T[\xi_i-\mu_i]\}\}\\
\hbox{[because $\xi_i\sim\SG(\mu_i,\rho^2_iI_m)$]\;\;}&\leq \sum_ix_i\exp\{ f^T[\mu_i-\mu]\} \exp\{\half\rho_i^2f^Tf\}\\
\hbox{[due to $\exp\{z\}\leq z+\exp\{\frac{2}3z^2\}\;\forall z$]\;\;} &\leq \exp\left\{\half\sigma^2f^Tf\right\}\sum_ix_i\left[ f^T[\mu_i-\mu]+\exp\{\tfrac{2}3[f^T[\mu-\mu_i]]^2\}\right]\\
&=\exp\left\{\half\sigma^2\|f\|_2^2\right\}\bigg[\sum_ix_i\exp\{\tfrac{2}3[f^T[\mu-\mu_i]]^2\}\bigg]
\end{align*}
due to $\mu=\sum_ix_i\mu_i$ and $\sum_ix_i=1$. On the other hand,
\begin{align*}
\sum_ix_i\exp\{\tfrac{2}3[f^T[\mu-\mu_i]]^2\}&=\sum_ix_i\exp\{\tfrac{2}3f^T[\mu-\mu_i][\mu-\mu_i]^Tf\}\\
&\leq\max_i\exp\{\tfrac{2}3\max_if^T[\mu-\mu_i][\mu-\mu_i]^Tf\}\\
&\leq\max_i\exp\{\tfrac{2}3\|\mu-\mu_i\|_2^2\|f\|_2^2\}\\
\hbox{\ [as $\mu\in\mathrm{Conv}\{\mu_1,...,\mu_n\}$]\;\;}&\leq\exp\{\tfrac{2}3\max_{i,j}\|\mu_i-\mu_j\|_2^2\|f\|_2^2\}.
\end{align*}
Thus, \rf{rho0} holds with
\[
\varrho
={2\over\sqrt{3}}\max_{i,j}\|\mu_i-\mu_j\|_2.\fqed
\]
\nocite{linform,juditsky2020polyhedral,PUP,JuNeATM}
 \end{document}